\documentclass[10pt,a4paper]{amsart} % For LaTeX2e
  
\usepackage[psamsfonts]{amssymb}
\usepackage[foot]{amsaddr}
\usepackage{amsmath}
\usepackage{amsthm}
\usepackage{paralist}
\usepackage{multirow}
\usepackage{bm}
\usepackage{graphics} 
\usepackage{graphicx}
\usepackage{pifont}
\usepackage{mathrsfs} 
\usepackage{algorithm} 
\usepackage{algorithmic} 
\usepackage{tabularx}
\usepackage{natbib}
\usepackage{appendix}
\usepackage[breaklinks]{hyperref}
\usepackage{wrapfig} 
\newenvironment{thmtag}[2][Theorem]{\begin{trivlist}
\item[\hskip \labelsep {\bfseries #1}\hskip \labelsep {\bfseries #2}]}{\end{trivlist}}

\newtheorem{mydef}{Definition}
\newtheorem{mythm}{Theorem} 

\newtheorem{mylem}[mythm]{Lemma}
\newtheorem{rem}{Remark}  

\providecommand{\abs}[1]{\lvert#1\rvert}
\providecommand{\norm}[1]{\lVert#1\rVert}
 
\bibpunct[; ]{(}{)}{,}{a}{}{;}

\newcommand{\bprob}{\ensuremath{\mathscr{P}^*}}   % the distribution on distributions
\newcommand{\inspace}{\ensuremath{\mathcal{X}}}   % the input space X
\newcommand{\sset}{\ensuremath{\mathcal{S}}}      % the sample set S
\newcommand{\outspace}{\ensuremath{\mathcal{Y}}}  % the output space Y
\newcommand{\pp}[1]{\ensuremath{\mathbb{#1}}}     % probability measure
\newcommand{\pspace}{\ensuremath{\mathfrak{P}}}   % probability space

\newcommand{\hbspace}{\ensuremath{\mathcal{H}}}   % Hilbert space on X
\newcommand{\hbspacey}{\ensuremath{\mathcal{F}}}  % Hilbert space on Y

\newcommand{\abbrvmm}[1]{\ensuremath{\mu_{#1}}}
\newcommand{\rd}[1]{\mathbb{R}^{#1}}              % the real space
\newcommand{\rr}{\mathbb{R}} 		          % the real numbers
\newcommand{\ep}{\mathbb{E}}                      % the expectation

\newcommand{\bopt}{\mathcal{B}}                   % the transformation operator
\newcommand{\bmat}{B}                    % the projection matrix V
\newcommand{\cmat}{C}                    % central subspace
\newcommand{\bvec}{\mathbf{b}}           % the projection vector v
\newcommand{\cvec}{\mathbf{c}}           % basis vectors of central subspace
\newcommand{\lmat}{L}                    % output kernel matrix
\newcommand{\qmat}{Q}                    % coefficient matrix 
\newcommand{\kxmat}{K}                   % the kernel matrix of X
\newcommand{\kymat}{\lmat}                 % the kernel matrix of Y
\newcommand{\id}{I}                      % the identity matrix I
                    % ...
\newcommand{\gmat}{G}                    % the gram matrix between distributions G
\newcommand{\covxx}{\Sigma_{\mathrm{xx}}}
\newcommand{\covxy}{\Sigma_{\mathrm{xy}}}
\newcommand{\covyy}{\Sigma_{\mathrm{yy}}}
\newcommand{\covyx}{\Sigma_{\mathrm{yx}}}

\newcommand{\dd}{\, \mathrm{d}}
  
\newcommand{\tick}{\ding{51}}
\newcommand{\cross}{\ding{55}}

\def\ci{\perp}

\begin{document}

\title{Domain Generalization via Invariant Feature Representation}

\date{\today}

\author[K. Muandet]{Krikamol Muandet}
\address{Max Planck Institute for Intelligent Systems, Spemannstra\ss e 38, 72076 T\"{u}bingen, Germany}
\email{krikamol@tuebingen.mpg.de}

\author[D. Balduzzi]{David Balduzzi}
\address{Department of Computer Science, ETH Zurich, Universit{\"a}tstrasse 6, 8092 Zurich, Switzerland}
\email{david.balduzzi@inf.ethz.ch}

\author[B. Sch\"{o}lkopf]{Bernhard Sch\"{o}lkopf}
\address{Max Planck Institute for Intelligent Systems, Spemannstra\ss e 38, 72076 T\"{u}bingen, Germany}
\email{bs@tuebingen.mpg.de}


\begin{abstract}
  This paper investigates domain generalization: How to take knowledge acquired from an arbitrary number of related domains and apply it to previously unseen domains? We propose Domain-Invariant Component Analysis (DICA), a kernel-based optimization algorithm that learns an invariant transformation by minimizing the dissimilarity across domains, whilst preserving the functional relationship between input and output variables. A learning-theoretic analysis shows that reducing dissimilarity improves the expected generalization ability of classifiers on new domains, motivating the proposed algorithm. Experimental results on synthetic and real-world datasets demonstrate that DICA successfully learns invariant features and improves classifier performance in practice.
\end{abstract}

\maketitle
 
%% introduction
\section{Introduction} 
  
Domain generalization considers how to take knowledge acquired from an arbitrary number of related domains, and apply it to previously unseen domains. To illustrate the problem, consider an example taken from \citet{Blanchard:11Generalizing} which studied automatic gating of flow cytometry data. For each of $N$ patients, a set of $n_i$ cells are obtained from peripheral blood samples using a flow cytometer. The cells are then labeled by an expert into different subpopulations, e.g., as a lymphocyte or not. Correctly identifying cell subpopulations is vital for diagnosing the health of patients. However, manual gating is very time consuming. To automate gating, we need to construct a classifier that generalizes well to previously unseen patients, where the distribution of cell types may differ dramatically from the training data. 
 
%A problem known as \emph{dataset bias} exhibits similar difficulty and has been extensively studied in the computer vision community (cf. \citep{Khosla12:Bias} and references therein). Many real-world image datasets have been widely used for both training and evaluating the object recognition system. Such datasets, however, are biased representative of the visual world; some of which capture more urban scenes, other more rural landscapes; some collected professional photographs, other the amateur snapshots from the Internet. These biases generally deteriorate the performance of the object recognition system when evaluated on different datasets. The model trained on the visual world would have the best cross-dataset generalization ability. Thus, good recognizers should learn how to discard these irrelevant biases, and focus on features that better reflects the images in the visual world.

Unfortunately, we cannot apply standard machine learning techniques directly because the data violates the basic assumption that training data and test data come from the same distribution. Moreover, the training set consists of heterogeneous samples from several distributions, i.e., gated cells from several patients. In this case, the data exhibits covariate (or dataset) shift \citep{widmer:96, quionero:09, bickel:09}: although the marginal distributions $\pp{P}_X$ on cell attributes vary due to biological or technical variations, the functional relationship $\pp{P}(Y|X)$ across different domains is largely stable (cell type is a stable function of a cell's chemical attributes). 

A considerable effort has been made in domain adaptation and transfer learning to remedy this problem, see \citet{Pan09:TKDE,ben-david:10} and references therein. Given a test domain, e.g., a cell population from a new patient, the idea of domain adaptation is to adapt a classifier trained on the training domain, e.g., a cell population from another patient, such that the generalization error on the test domain is minimized. The main drawback of this approach is that one has to repeat this process for every new patient, which can be time-consuming -- especially in medical diagnosis where time is a valuable asset. In this work, across-domain information, which may be more informative than the domain-specific information, is extracted from the training data and used to generalize the classifier to new patients without retraining. 

%% domain generalization diagram
\begin{figure}[t!]
  \centering
  \includegraphics[width=0.7\linewidth]{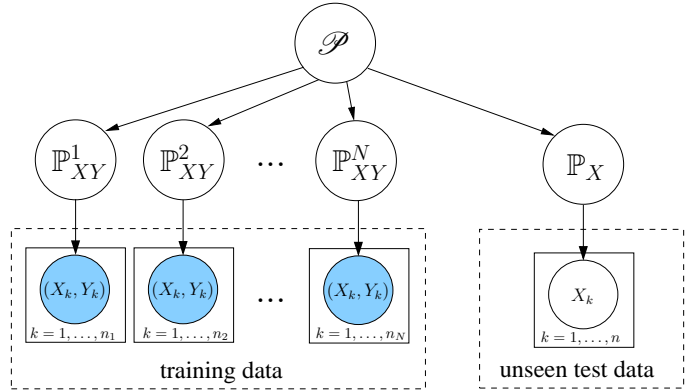}
  \caption{A simplified schematic diagram of the domain generalization framework. A major difference between our framework and most previous work in domain adaptation is that we do not observe the test domains during training time. See text for detailed description on how the data are generated.}
  \label{fig:dg-diagram}
\end{figure}

\subsection{Overview.}
The goal of (supervised) domain generalization is to estimate a functional relationship that handles changes in the marginal $\pp{P}(X)$ or conditional $\pp{P}(Y|X)$ well, see Figure \ref{fig:dg-diagram}. We assume that the conditional probability $\pp{P}(Y|X)$ is stable or varies smoothly with the marginal $\pp{P}(X)$. Even if the conditional is stable, learning algorithms may still suffer from \emph{model misspecification} due to variation in the marginal $\pp{P}(X)$. That is, if the learning algorithm cannot find a solution that perfectly captures the functional relationship between $X$ and $Y$ then its approximate solution will be sensitive to changes in $\pp{P}(X)$. 

In this paper, we introduce Domain Invariant Component Analysis (DICA), a kernel-based algorithm that finds a transformation of the data that {\rm (i)} minimizes the difference between marginal distributions $\pp{P}_X$ of domains as much as possible while {\rm (ii)} preserving the functional relationship $\pp{P}(Y|X)$. 

The novelty of this work is twofold. First, DICA extracts \emph{invariants}: features that transfer across domains. It not only minimizes the divergence between marginal distributions $\pp{P}(X)$, but also preserves the functional relationship encoded in the posterior $\pp{P}(Y|X)$.  The resulting learning algorithm is very simple. Second, while prior work in domain adaptation focused on using data from many different domains to specifically improve the performance on the target task, which is observed during the training time (the classifier is adapted to the specific target task), we assume access to abundant training data and are interested in the generalization ability of the invariant subspace to previously unseen domains (the classifier generalizes to new domains without retraining).

Moreover, we show that DICA generalizes or is closely related to many well-known dimension reduction algorithms including kernel principal component analysis (KPCA) \citep{Scholkopf98:KPCA, Fukumizu04:KDR}, %kernel canonical correlation analysis (KCCA) \citep{Hardoon04:CCA}, 
transfer component analysis (TCA) \citep{Pan11:TCA}, and covariance operator inverse regression (COIR) \citep{Kim11:COIR}, see \S\ref{sec:relations}. The performance of DICA is analyzed theoretically \S\ref{sec:theory} and demonstrated empirically \S\ref{sec:experiments}. 

\subsection{Related work.}
Domain generalization is a form of transfer learning, which applies expertise acquired in source domains to improve learning of target domains (cf. \citet{Pan09:TKDE} and references therein). Most previous work assumes the availability of the target domain to which the knowledge will be transferred. In contrast, domain generalization focuses on the generalization ability on previously unseen domains. That is, the test data comes from domains that are not available during training. 

Recently, \citet{Blanchard:11Generalizing} proposed an augmented SVM that incorporates empirical marginal distributions into the kernel. A detailed error analysis showed universal consistency of the approach. We apply methods from \citet{Blanchard:11Generalizing} to derive theoretical guarantees on the finite sample performance of DICA.
 
Learning a shared subspace is a common approach in settings where there is distribution mismatch. For example, a typical approach in multitask learning is to uncover a joint (latent) feature/subspace that benefits tasks individually \citep{Argyriou07:MTL, Gu09:LSS, Passos12:FMLTS}. A similar idea has been adopted in domain adaptation, where the learned subspace reduces mismatch between source and target domains \citep{Gretton09:KMM, Pan11:TCA}. Although these approaches have proven successful in various applications, no previous work has fully investigated the generalization ability of a subspace to unseen domains.

%% stationary component analysis
\section{Domain-Invariant Component Analysis}
\label{sec:sca}

Let $\inspace$ denote a nonempty input space and $\outspace$ an arbitrary output space. We define a \textbf{domain} to be a joint distribution $\pp{P}_{XY}$ on $\inspace\times\outspace$, and let $\mathfrak{P}_{\inspace\times\outspace}$ denote the set of all domains. Let $\mathfrak{P}_{\inspace}$ and $\mathfrak{P}_{\outspace | \inspace}$ denote the set of probability distributions $\pp{P}_X$ on $X$ and $\pp{P}_{Y|X}$ on $Y$ given $X$ respectively. %We assume that the disintegration theorem [ref] holds, so that any element $\pp{P}_{XY}\in\mathfrak{P}_{\inspace\times\outspace}$ decomposes into a product $\pp{P}_{XY}=\pp{P}_X\cdot\pp{P}_{Y|X}$. %The space $\mathfrak{P}_{\inspace\times\outspace}$ is endowed with the topology of weak convergence and the associated Borel $\sigma$-algebra. 

We assume domains are sampled from probability distribution $\mathscr{P}$ on $\mathfrak{P}_{\inspace\times\outspace}$ which has a bounded second moment, i.e., the variance is well-defined. Domains are not observed directly. Instead, we observe $N$ samples $\sset=\{S^i\}_{i=1}^N$, where $S^i=\{(x^{(i)}_k,y^{(i)}_k)\}_{k=1}^{n_i}$ is sampled from $\pp{P}^{i}_{XY}$ and each $\pp{P}^{1}_{XY},\ldots,\pp{P}^{N}_{XY}$ is sampled from $\mathscr{P}$. Since in general $\pp{P}_{XY}^{i}\neq \pp{P}_{XY}^{j}$, the samples in $\sset$ are not i.i.d. Let $\widehat{\pp{P}}^i$ denote empirical distribution associated with each sample $S^i$. For brevity, we use $\pp{P}$ and $\pp{P}_X$ interchangeably to denote the marginal distribution.

Let $\hbspace$ and $\hbspacey$ denote reproducing kernel Hilbert spaces (RKHSes) on $\inspace$ and $\outspace$ with kernels $k:\inspace\times\inspace\rightarrow\rr$ and $l:\outspace\times\outspace\rightarrow\rr$, respectively. Associated with $\hbspace$ and $\hbspacey$ are mappings $x\rightarrow\phi(x)\in\hbspace$ and $y\rightarrow\varphi(y)\in\hbspacey$ induced by the kernels $k(\cdot,\cdot)$ and $l(\cdot,\cdot)$. Without loss of generality, we assume the feature maps of $X$ and $Y$ have zero means, i.e., $\sum_{k=1}^n\phi(x_k)=0=\sum_{k=1}^n\varphi(y_k)$. Let $\covxx$, $\covyy$, $\covxy$, and $\covyx$ be the covariance operators in and between the RKHSes of $X$ and $Y$.

\subsection{Objective.}
Using the samples $\sset$, our goal is to produce an estimate $f:\mathfrak{P}_{\inspace}\times\inspace\rightarrow\rr$ that generalizes well to test samples $S^t=\{x^{(t)}_k\}_{k=1}^{n_t}$ drawn according to some unknown distribution $\pp{P}^{t}\in\mathfrak{P}_{\inspace}$ \citep{Blanchard:11Generalizing}. Since the performance of $f$ depends in part on how dissimilar the test distribution $\pp{P}^{t}$ is from those in the training samples, we propose to preprocess the data to actively reduce the dissimilarity between domains. Intuitively, we want to find transformation $\bopt$ in $\hbspace$ that 
\begin{inparaenum}[(i)] 
\item minimizes the distance between empirical distributions %$\widehat{\pp{P}}^{1},\widehat{\pp{P}}^{2},\ldots,\widehat{\pp{P}}^{N}$ 
of the transformed samples $\bopt(S^i)$ and 
\item preserves the functional relationship between $X$ and $Y$, i.e., $Y \ci X \,|\, \bopt(X)$.
\end{inparaenum}
We formulate an optimization problem capturing these constraints below.

%% formal setting
\subsection{Distributional Variance}
\label{sec:dist-variance}
 
First, we define the distributional variance, which measures the dissimilarity across domains. It is convenient to represent distributions as elements in an RKHS \citep{Bertinet04:RKHS, Smola07Hilbert, Sriperumbudur10:Metrics} using the \textbf{mean map}
\begin{equation}
  \label{eq:meanmap}
  \mu : \; \pspace_{\inspace} \rightarrow \hbspace: \; \pp{P} \mapsto \int_\inspace k(x,\cdot)\dd\pp{P}(x) =:\mu_{\pp{P}}\enspace .
\end{equation} 
We assume that $k(x,x)$ is bounded for any $x\in\inspace$ such that $\ep_{x\sim\pp{P}}[k(x,\cdot)] < \infty$. If $k$ is characteristic then \eqref{eq:meanmap} is injective, i.e., all the information about the distribution is preserved \citep{Sriperumbudur10:Metrics}. It also holds that $\ep_{\mathbb{P}}[f]=\langle\abbrvmm{\pp{P}}, f\rangle_{\hbspace}$ for all $f\in\hbspace$ and any $\pp{P}$. 
 
We decompose $\mathscr{P}$ into $\mathscr{P}_X$, which generates the marginal distribution $\pp{P}_X$, and $\mathscr{P}_{Y|X}$, which generates posteriors $\pp{P}_{Y|X}$. The data generating process begins by generating the marginal $\pp{P}_X$ according to $\mathscr{P}_X$. Conditioned on $\pp{P}_X$, it then generate conditional $\pp{P}_{Y|X}$ according to $\mathscr{P}_{Y|X}$. The data point $(x,y)$ is generated according to $\pp{P}_X$ and $\pp{P}_{Y|X}$, respectively. Given set of distributions ${\mathcal P}=\{\pp{P}^1,\pp{P}^2\dotsc,\pp{P}^N\}$ drawn according to $\mathscr{P}_X$, define $N\times N$ Gram matrix $\gmat$ with entries
\begin{eqnarray}
  \label{eq:expected-kernel}
  \gmat_{ij} := \langle\mu_{\pp{P}^i},\mu_{\pp{P}^j}\rangle_{\hbspace} = \iint k(x,z) \dd\pp{P}^i(x) \dd\pp{P}^j(z), 
\end{eqnarray}
\noindent for $i,j=1,\dotsc,N$. Note that $\gmat_{ij}$ is the inner product between kernel mean embeddings of $\pp{P}^i$ and $\pp{P}^j$ in $\hbspace$. Based on \eqref{eq:expected-kernel}, we define the distributional variance, which estimates the variance of the distribution $\mathscr{P}_X$:

\begin{mydef}
  Introduce probability distribution ${\mathcal P}$ on $\hbspace$ with ${\mathcal P}(\mu_{\pp{P}^i})=\frac{1}{N}$ and center $\gmat$ to obtain the covariance operator of ${\mathcal P}$, denoted as $\Sigma :=\gmat - \mathbf{1}_{N}\gmat - \gmat\mathbf{1}_{N}  + \mathbf{1}_{N}\gmat\mathbf{1}_{N}$. The \textbf{distributional variance} is
  \begin{equation}
    \label{eq:tr-variance}  
    {\mathbb V}_{\hbspace}({\mathcal P}) := \frac{1}{N}\mathrm{tr}(\Sigma) 
    = \frac{1}{N}\mathrm{tr}(\gmat) - \dfrac{1}{N^2}\sum_{i,j=1}^{N}\gmat_{ij}.
  \end{equation}  
\end{mydef}
 
The following theorem shows that the distributional variance is suitable as a measure of divergence between domains. 
% theorem of RKHS-based variance
\begin{mythm}
  \label{thm:uniqueness}
  Let $\bar{\pp{P}}=\frac{1}{N}\sum_{i=1}^{N}\pp{P}^i$. If $k$ is a characteristic kernel, then $\mathbb{V}_{\hbspace}(\mathcal{P})=\frac{1}{N}\sum_{i=1}^{N}\norm{\abbrvmm{\pp{P}^i} - \abbrvmm{\bar{\pp{P}}}}_{\hbspace}^2=0$  if and only if $ \pp{P}^1 = \pp{P}^2 = \cdots=\pp{P}^N$.
\end{mythm}

To estimate ${\mathbb V}_{\hbspace}({\mathcal P})$ from $N$ sample sets $\sset=\{S^i\}_{i=1}^N$ drawn from $\pp{P}^1,\ldots,\pp{P}^N$, we define block kernel and coefficient matrices
\begin{equation*}
  \kxmat = \left(
    \begin{array}{cccc}
      K_{1,1} & \cdots & K_{1,N} \\
      \vdots  & \ddots & \vdots \\
      K_{N,1} & \cdots & K_{N,N}
    \end{array}
  \right) \in \mathbb{R}^{n \times n} \enspace , 
  \qmat = \left(
    \begin{array}{cccc}
      \qmat_{1,1} & \cdots & \qmat_{1,N} \\
      \vdots  & \ddots & \vdots \\
      \qmat_{N,1} & \cdots & \qmat_{N,N}
    \end{array}
  \right) \in \mathbb{R}^{n \times n} \enspace ,
\end{equation*}
where $n=\sum_{i=1}^{N}n_i$ and $[K_{i,j}]_{k,l}=k(x^{(i)}_k,x^{(j)}_l)$ is the Gram matrix evaluated between the sample $S^i$ and $S^j$. Following \eqref{eq:tr-variance}, elements of the coefficient matrix $\qmat_{i,j}\in\mathbb{R}^{n_i\times n_j}$ equal $(N-1)/(N^2n_i^2)$ if $i=j$, and $-1/(N^2 n_i n_j)$ otherwise. Hence, the empirical distributional variance is  
\begin{equation}
\widehat{\mathbb{V}}_{\hbspace}(\sset) = \frac{1}{N}\mathrm{tr}(\widehat{\Sigma}) = \mathrm{tr}(\kxmat\qmat) \enspace .
\end{equation}

\begin{mythm}
  \label{thm:vestimator}
  The empirical estimator $\widehat{\mathbb{V}}_{\hbspace}(\sset) = \frac{1}{N}\mathrm{tr}(\widehat{\Sigma}) = \mathrm{tr}(\kxmat\qmat)$ obtained from Gram matrix
  \begin{equation*}
    %\label{eq:empirical-kernel}
    \widehat{\gmat}_{ij} 
   :=\dfrac{1}{n_i\cdot n_j}\sum_{k=1}^{n_i}\sum_{l=1}^{n_j}
    k(x^{(i)}_k,x^{(j)}_l)
  \end{equation*}
  is a consistent estimator of $\mathbb{V}_{\hbspace}({\mathcal P})$.
\end{mythm} 

\subsection{Formulation of DICA}
DICA finds an orthogonal transform $\bopt$ onto a low-dimensional subspace ($m \ll n$) that minimizes the distributional variance ${\mathbb V}_{\hbspace}(\sset)$ between samples from $\sset$, i.e. the \emph{dissimilarity across domains}. Simultaneously, we require that $\bopt$ preserves the functional relationship between $X$ and $Y$, i.e. $Y \ci X \,|\, \bopt(X)$.

\subsubsection{Minimizing distributional variance.}
In order to simplify notation, we ``flatten'' $\{(x^{(i)}_k,y^{(i)}_k)_{k=1}^{n_i}\}_{i=1}^{N}$ to $\{(x_k,y_k)\}_{k=1}^n$ where $n=\sum_{i=1}^{N}n_i$. Let $\bvec_k = \sum_{i=1}^{n}\beta^i_k\phi(x_i) = \Phi_x\bm{\beta}_k$ be the $k^{th}$ basis function of $\bopt$ where $\Phi_x=[\phi(x_1),\phi(x_2),\dotsc,\phi(x_n)]$ and $\bm{\beta}_k$ are $n$-dimensional coefficient vectors. Let $B=[\bm{\beta}_1,\bm{\beta}_2,\ldots,\bm{\beta}_m]$ and $\tilde{\Phi}_x$ denote the projection of $\Phi_x$ onto $\bvec_k$, i.e., $\tilde{\Phi}_x=\bvec_k^{\top}\Phi_x=\bm{\beta}_k^{\top}\Phi_x^{\top}\Phi_x = \bm{\beta}_k^{\top}\kxmat$. The kernel on the $\bopt$-projection of $X$ is
\begin{equation}
	\label{e:bkernel}
  \widetilde{K} := \tilde{\Phi}_x^{\top}\tilde{\Phi}_x = \kxmat BB^{\top}\kxmat \enspace.
\end{equation} 
After applying transformation $\bopt$, the empirical distributional variance between sample distributions is
\begin{equation}
  \widehat{\mathbb{V}}_{\hbspace}(\bopt\sset) = \mathrm{tr}(\widetilde{K} \qmat)
  %=\mathrm{tr}(\kxmat BB^{\top}\kxmat\qmat)
  =\mathrm{tr}(B^{\top}\kxmat\qmat\kxmat B) \enspace .   \label{eq:dvariance}
\end{equation}
%Hence, DICA finds $B$ that minimizes \eqref{eq:dvariance}.

\subsubsection{Preserving the functional relationship.}
The central subspace $C$ is the minimal subspace that captures the functional relationship between $X$ and $Y$, i.e. $Y\ci X|\cmat^\top X$. Note that in this work we generalize a linear transformation $\cmat^\top X$ to nonlinear one $\bopt(X)$. To find the central subspace we use the inverse regression framework, \citep{Li91:Sliced}: %Assume without loss of generality that $X$ is centered, i.e. $\ep[X]=0$.

\begin{mythm}
  \label{thm:li}
     If there exists a \textbf{central subspace} $\cmat=[ \cvec_1,\ldots, \cvec_m]$ satisfying $Y\ci X|\cmat^\top X$, and for any $a\in\rd{d}$, $\ep[a^\top X|\cmat^\top X]$ is linear in $\{\cvec^\top_i X\}_{i=1}^m$, then $\ep[X|Y]\subset\text{span}\{\covxx\cvec_i\}_{i=1}^m$.   
\end{mythm}

It follows that the bases $\cmat$ of the central subspace coincide with the $m$ largest eigenvectors of $\mathbb{V}(\ep[X|Y])$ premultiplied by $\covxx^{-1}$. Thus, the basis $\cvec$ is the solution to the eigenvalue problem $\mathbb{V}(\ep[X|Y])\covxx\cvec = \gamma\covxx\cvec$. Alternatively, for each $\cvec_k$ one may solve
\begin{align*}
  \underset{\cvec_k\in\rr^d}{\text{max}} \; 
\frac{\cvec_k^{\top}\covxx^{-1}\mathbb{V}(\ep[X|Y])\covxx\cvec_k}{\cvec_k^{\top}\cvec_k}
\end{align*}
under the condition that $\cvec_k$ is chosen to not be in the span of the previously chosen $\cvec_k$. In our case, $x$ is mapped to $\phi(x)\in\hbspace$ induced by the kernel $k$ and $\bopt$ has nonlinear basis functions $\cvec_k\in\hbspace,\, k=1,\ldots,m$. This nonlinear extension implies that $\ep[X|Y]$ lies on a function space spanned by $\{\covxx\cvec_k\}_{k=1}^m$, which coincide with the eigenfunctions of the operator $\mathbb{V}(\ep[X|Y])$ \citep{Wu2008:KSIR,Kim11:COIR}. Since we always work in $\hbspace$, we drop $\phi$ from the notation below.

To avoid slicing the output space explicitly \citep{Li91:Sliced,Wu2008:KSIR}, we exploit its kernel structure when estimating the covariance of the inverse regressor. The following result from \citet{Kim11:COIR} states that, under a mild assumption, $\mathbb{V}(\ep[X|Y])$ can be expressed in terms of covariance operators:
%%% theorem
\begin{mythm}
  \label{thm:cov-optr}
  If for all $f\in\hbspace$, there exists $g\in\hbspacey$ such that $\ep[f(X)|y]=g(y)$ for almost every $y$, then 
  \begin{equation}
    \label{eq:covariance-ir}
    \mathbb{V}(\ep[X|Y]) = \covxy\covyy^{-1}\covyx \enspace .
  \end{equation}
\end{mythm}

Let $\Phi_y=[\varphi(y_1),\ldots,\varphi(y_n)]$ and $\lmat=\Phi_y^{\top}\Phi_y$. The covariance of inverse regressor \eqref{eq:covariance-ir} is estimated from the samples $\sset$ as $\widehat{\mathbb{V}}(\ep[X|Y]) = \widehat{\Sigma}_{\mathrm{xy}}\widehat{\Sigma}_{\mathrm{yy}}^{-1}\widehat{\Sigma}_{\mathrm{yx}}=\frac{1}{n}\Phi_x\lmat(\lmat + n\varepsilon\id_n)^{-1}\Phi_x^{\top}$ where $\widehat{\Sigma}_{\mathrm{xy}} = \frac{1}{n}\Phi_x\Phi_y^{\top}$ and $\widehat{\Sigma}_{\mathrm{yy}} = \frac{1}{n}\Phi_y\Phi_y^{\top}$. Assuming inverses $\widehat{\Sigma}_{\mathrm{yy}}^{-1}$ and $\widehat{\Sigma}_{\mathrm{xx}}^{-1}$ exist, a straightforward computation (see Supplementary) shows
\begin{eqnarray} 	\bvec_k^{\top}\widehat{\Sigma}_{\mathrm{xx}}^{-1}\widehat{\mathbb{V}}(\ep[X|Y])\widehat{\Sigma}_{\mathrm{xx}}\bvec_k &=& \frac{1}{n}\bm{\beta}_k^{\top}\lmat(\lmat+n\varepsilon\id)^{-1}\kxmat^2\bm{\beta}_k \nonumber \\
	\bvec_k^{\top}\bvec_k &=& \bm{\beta}_k^{\top}\kxmat\bm{\beta}_k,
	\label{eq:flda}
\end{eqnarray}
\noindent where $\varepsilon$ smoothes the affinity structure of the output space $Y$, thus acting as a kernel regularizer. Since we are interested in the projection of $\phi(x)$ onto the basis functions $\bvec_k$, we formulate the optimization in terms of $\bm{\beta}_k$. For a new test sample $x_t$, the projection onto basis function $\bvec_k$ is $\mathbf{k}_t\bm{\beta}_k$, where $\mathbf{k}_t=[k(x_1,x_t),\ldots,k(x_n,x_t)]$. 

\subsubsection{The optimization problem.}
Combining \eqref{eq:dvariance} and \eqref{eq:flda}, DICA finds $B=[\bm{\beta}_1,\bm{\beta}_2,\ldots,\bm{\beta}_m]$ that solves
\begin{equation}
	\label{eq:ssca}
    \underset{\bmat\in\rr^{n\times m}}{\text{max}}
	\frac{\frac{1}{n}\mathrm{tr}\left(\bmat^{\top}\lmat(\lmat + n\varepsilon\id_n)^{-1}\kxmat^2\bmat\right)}
	{\mathrm{tr}\left(\bmat^{\top}\kxmat\qmat\kxmat\bmat + \bmat\kxmat\bmat\right)}
\end{equation}

The numerator requires that $\bmat$ aligns with the bases of the central subspace. The denominator forces both dissimilarity across domains and the complexity of $\bmat$ to be small, thereby tightening generalization bounds, see \S\ref{sec:theory}. Rewriting \eqref{eq:ssca} as a constrained optimization (see Supplementary) yields Lagrangian
\begin{align}
  \mathcal{L} =& \frac{1}{n}\mathrm{tr}\left(\bmat^{\top}\lmat(\lmat + n\varepsilon\id_n)^{-1}\kxmat^2\bmat\right) - \mathrm{tr}\left(\left(\bmat^{\top}\kxmat\qmat\kxmat\bmat + \bmat\kxmat\bmat - \id_m\right)\Gamma\right) \label{eq:ssca-lagragian} \enspace ,
\end{align}
\noindent where $\Gamma$ is a diagonal matrix containing the Lagrange multipliers. Setting the derivative of \eqref{eq:ssca-lagragian} w.r.t. $B$ to zero yields the generalized eigenvalue problem:
\begin{equation}
  \frac{1}{n}\lmat(\lmat + n\varepsilon\id_n)^{-1}\kxmat^2 B
  = (\kxmat\qmat\kxmat + \kxmat)B \Gamma \enspace .
  \label{eq:bsolution} 
\end{equation}
Transformation $B$ corresponds to the $m$ leading eigenvectors of the generalized eigenvalue problem \eqref{eq:bsolution}\footnote{In practice, it is more numerically stable to solve the generalized eigenvalue problem $\frac{1}{n}\lmat(\lmat + n\varepsilon\id_n)^{-1}\kxmat^2 B = (\kxmat\qmat\kxmat + \kxmat + \lambda\id)B \Gamma$, where $\lambda$ is a small constant.}.

The inverse regression framework based on covariance operators has two benefits. First, it avoids explicitly slicing the output space, which makes it suitable for high-dimensional output. Second, it allows for structured outputs on which explicit slicing may be impossible, e.g., trees and sequences. Since our framework is based entirely on kernels, it is applicable to any type of input and output variables, as long as the corresponding kernels can be defined.

% unsupervised DICA
\subsection{Unsupervised DICA}
\label{sec:usca} 

In some application domains, such as image denoising, information about the target may not be available. We therefore derive an unsupervised version of DICA. Instead of preserving the central subspace, unsupervised DICA (UDICA) maximizes the variance of $X$ in the feature space, which is estimated as $\frac{1}{n}\mathrm{tr}(B^{\top}\kxmat^2B)$. Thus, UDICA solves
\begin{equation}
  \label{eq:udica}
  \underset{B\in\rr^{n\times m}}{\text{max}}
  \frac{\frac{1}{n}\mathrm{tr}(B^{\top}\kxmat^2 B)}
  {\mathrm{tr}(B^{\top}\kxmat\qmat\kxmat B + B^{\top}\kxmat B)}
 \enspace . 
\end{equation}
Similar to DICA, the solution of \eqref{eq:udica} is obtained by solving the generalized eigenvalue problem 
\begin{equation}
  \frac{1}{n}\kxmat^2B = (\kxmat\qmat\kxmat + \kxmat)B\Gamma \enspace . 
\end{equation}
UDICA is a special case of DICA where $\lmat = \frac{1}{n}\id$ and $\varepsilon\rightarrow 0$. Algorithm \ref{alg:dica} summarizes supervised and unsupervised domain-invariant component analysis.

%%%%%%%%%%%%%%%%%%%%%%%%%%%%%%%%%%%%%%%%%%%%%%%%%%%%%%%%%%%%%%%%%%%%%%%%%%%%%%%%%%%%%%%%%%%%%%%%%%%%%%
%% the step-by-step DICA algorithm
\begin{algorithm}[t!]
  \caption{Domain-Invariant Component Analysis}
  \begin{tabular}{ll}
  \textbf{Input:} & Parameters $\lambda$, $\varepsilon$, and $m \ll n$. Sample $\sset=\{S^i=\{(x_k^{(i)},y_k^{(i)})\}_{k=1}^{n_i}\}_{i=1}^N$. \\
  \textbf{Output:} & Projection $\bmat_{n\times m}$ and kernel $\widetilde{K}_{n\times n}$.
  \end{tabular}
  \begin{algorithmic}[1]
    \STATE Calculate gram matrix $[\kxmat_{ij}]_{kl}=k(x_k^{(i)},x_l^{(j)})$ and $[\kymat_{ij}]_{kl}=l(y_k^{(i)},y_l^{(j)})$.
    \STATE \textbf{Supervised:} $C=\kymat(\kymat + n\varepsilon\id)^{-1}\kxmat^2$.
    \STATE \textbf{Unsupervised:} $C=\kxmat^2$.
    \STATE Solve $\frac{1}{n}CB = (\kxmat\qmat\kxmat + \kxmat+\lambda\id)B\Gamma$ for $B$.
    \STATE Output $\bmat$ and  $\widetilde{K} \leftarrow \kxmat\bmat\bmat^{\top}\kxmat$. 
    \STATE The test kernel $\widetilde{K}^t \leftarrow \kxmat^t\bmat\bmat^{\top}\kxmat$ where $\kxmat^t_{n_t\times n}$ is the joint kernel between test and training data.
  \end{algorithmic}
  \label{alg:dica}
\end{algorithm}
%%%%%%%%%%%%%%%%%%%%%%%%%%%%%%%%%%%%%%%%%%%%%%%%%%%%%%%%%%%%%%%%%%%%%%%%%%%%%%%%%%%%%%%%%%%%%%%%%%%%%% 

\subsection{Relations to Other Methods}
\label{sec:relations}

The DICA and UDICA algorithms generalize many well-known dimension reduction techniques. In the supervised setting, if dataset $\sset$ contains samples drawn from a single distribution $\pp{P}_{XY}$ then we have $\kxmat\qmat\kxmat=\mathbf{0}$. Substituting $\alpha:= KB$ gives the eigenvalue problem $\frac{1}{n}\kymat(\kymat + n\varepsilon\id)^{-1}\kxmat \alpha = \kxmat \alpha \Gamma$, which corresponds to covariance operator inverse regression (COIR) \citep{Kim11:COIR}.

If there is only a single distribution then unsupervised DICA reduces to KPCA since  $\kxmat\qmat\kxmat = \mathbf{0}$ and finding $\bmat$ requires solving the eigensystem $\kxmat\bmat = \bmat\Gamma$ which recovers   KPCA \citep{Scholkopf98:KPCA}. If there are two domains, source $\pp{P}_S$ and target $\pp{P}_T$, then UDICA is closely related -- though not identical to -- Transfer Component Analysis \citep{Pan11:TCA}. This follows from the observation that $\mathbb{V}_{\hbspace}(\{\pp{P}_S,\pp{P}_T\}) = \|\mu_{\pp{P}_S}-\mu_{\pp{P}_T}\|^2$, see proof of Theorem~\ref{thm:uniqueness}.

%% theory
\subsection{A Learning-Theoretic Bound}
\label{sec:theory}

We bound the generalization error of a classifier trained after DICA-preprocessing. The main complication is that samples are not identically distributed. We adapt an approach to this problem developed in  \citet{Blanchard:11Generalizing} to prove a generalization bound that applies after transforming the empirical sample using $\bopt$. Recall that $\bopt = \Phi_x\bmat$.

Define kernel $\bar{k}$ on $\pspace\times \inspace$ as $\bar{k}((\pp{P},x),(\pp{P}',x')):=k_\pspace(\pp{P},\pp{P}')\cdot k_\inspace(x,x')$. Here, $k_\inspace$ is the kernel on $\hbspace_\inspace$ and the kernel on distributions is $k_\pspace(\pp{P},\pp{P}') := \kappa(\mu_{\pp{P}},\mu_{\pp{P}'})$ where $\kappa$ is a positive definite kernel \citep{Christmann10:Universal, Muandet12:SMM}. Let $\Psi_\pspace$ denote the corresponding feature map.

\begin{mythm}\label{t:lbound}
	Under reasonable technical assumptions, see Supplementary, it holds with probability at least $1-\delta$ that,
	\begin{gather*}
		\sup_{\|f\|_{\hbspace}\leq 1} 
	\left|\ep_\bprob \ep_\pp{P}\ell(f(\tilde{X}_{ij}\bopt),Y_i) 
	-\ep_{\hat{\pp{P}}} \ell(f(\tilde{X}_{ij}\bopt),Y_i)\right|^2\\
	  	 \leq  c_1 \frac{1}{N} \mathrm{tr}(\bmat^\intercal \kxmat\qmat\kxmat\bmat) 
		+ \mathrm{tr}(\bmat^\top\kxmat\bmat)\left(c_2\frac{N(\log\frac{1}{\delta}+2\log N)}{n}
		+ \frac{c_3\log\frac{1}{\delta}+c_4}{N}\right).
	\end{gather*}
\end{mythm}

The LHS is the difference between the training error and expected error (with respect to the distribution on domains $\bprob$) after applying $\bopt$. 

The first term in the bound, involving $\text{tr}(\bmat^\intercal \kxmat\qmat\kxmat\bmat)$, quantifies the distributional variance after applying the transform: the higher the distributional variance, the worse the guarantee, tying in with analogous results in \citet{Ben-david07:DAF, ben-david:10}. The second term in the bound depends on the size of the distortion ${\rm tr}(\bmat^\intercal\kxmat\bmat)$ introduced by $\bmat$: the more complicated the transform, the worse the guarantee.

The bound reveals a tradeoff between reducing the distributional variance and the complexity or size of the transform used to do so. The denominator of \eqref{eq:ssca} is a sum of these terms, so that DICA tightens the bound in Theorem~\ref{t:lbound}.

Preserving the functional relationship (i.e. central subspace) by maximizing the numerator in \eqref{eq:ssca} should reduce the empirical risk $\ep_{\hat{\pp{P}}} \ell(f(\tilde{X}_{ij}\bopt),Y_i)$. However, a rigorous demonstration has yet to be found.

%% experiments
\section{Experiments}
\label{sec:experiments}

We illustrate the difference between the proposed algorithms and their single-domain counterparts using a synthetic dataset. Furthermore, we evaluate DICA in two tasks: a classification task on flow cytometry data and a regression task for Parkinson's telemonitoring.

\subsection{Toy Experiments}

We generate 10 collections of $n_i\sim\mathtt{Poisson}(200)$ data points. The data in each collection is generated according to a five-dimensional zero-mean Gaussian distribution. For each collection, the covariance of the distribution is generated from Wishart distribution $\mathcal{W}(0.2\times\id_5,10)$. This step is to simulate different marginal distributions. The output value is $y=\mathrm{sign}(b_1^{\top}x + \epsilon_1)\cdot\log(\abs{b_2^{\top}x + c + \epsilon_2})$, where $b_1,b_2$ are the weight vectors, $c$ is a constant, and $\epsilon_1,\epsilon_2\sim\mathcal{N}(0,1)$. Note that $b_1$ and $b_2$ form a low-dimensional subspace that captures the functional relationship between $X$ and $Y$. We then apply the KPCA, UDICA, COIR, and DICA algorithms on the dataset with Gaussian RBF kernels for both $X$ and $Y$ with bandwidth parameters $\sigma_x=\sigma_y=1$, $\lambda=0.1$, and $\varepsilon=10^{-4}$.  

Fig.~\ref{fig:synthetic} shows projections of the training and three previously unseen test datasets onto the first two eigenvectors. The subspaces obtained from UDICA and DICA are more stable than for KPCA and COIR. In particular, COIR shows a substantial difference between training and test data, suggesting overfitting.

%%%%%%%%%%%%%%%%%%%%%%%%%%%%%%%%%%%%%%%%%%%%%%%%%%%%
%% results of toy experiments 
\begin{figure*}[t]
  \centering
  \begin{tabular}{c}
    \includegraphics[width=0.7\linewidth]{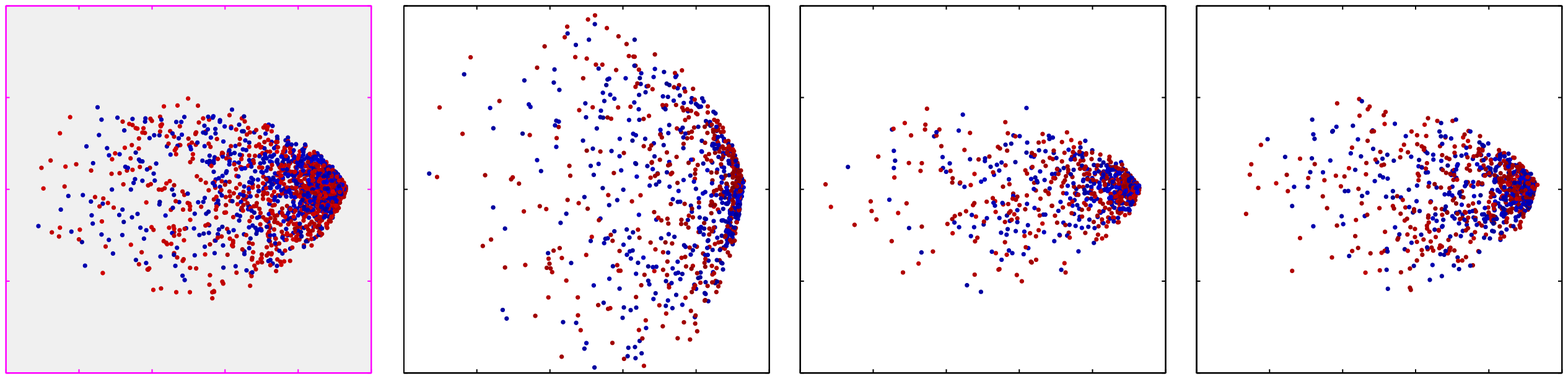} \\ 
    KPCA \\
    \includegraphics[width=0.7\linewidth]{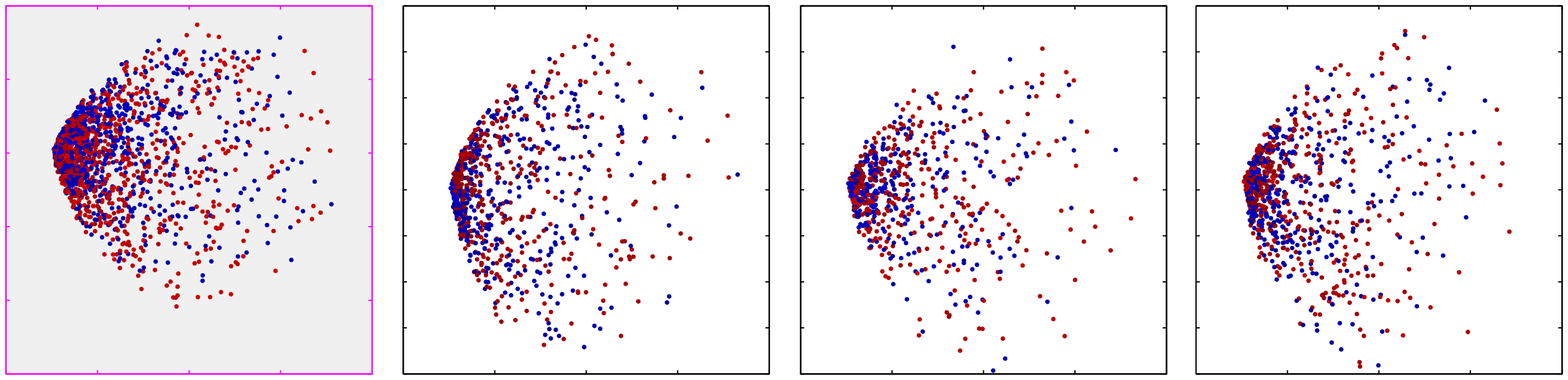} \\
    UDICA \\
    \includegraphics[width=0.7\linewidth]{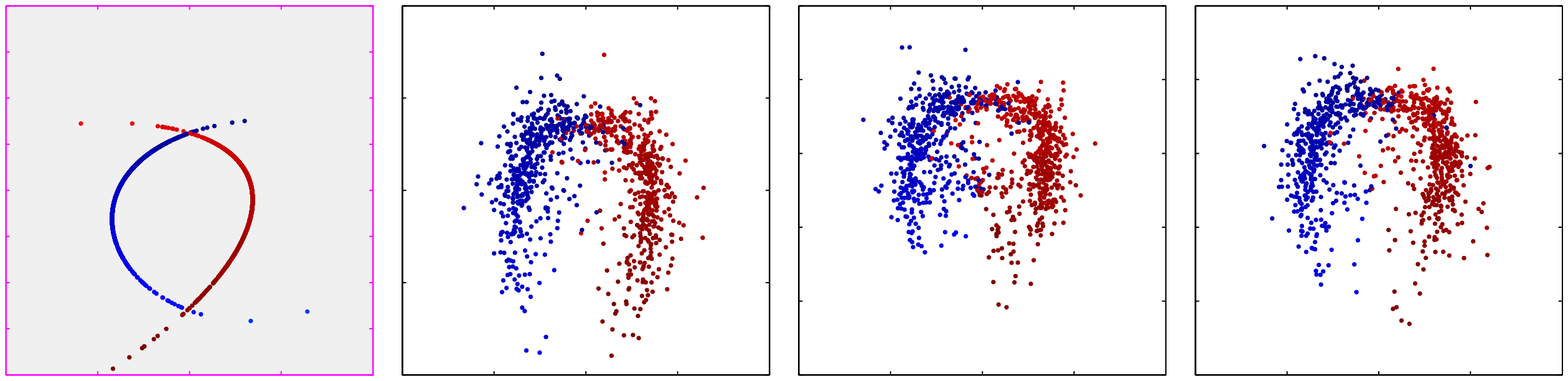} \\
    COIR \\
    \includegraphics[width=0.7\linewidth]{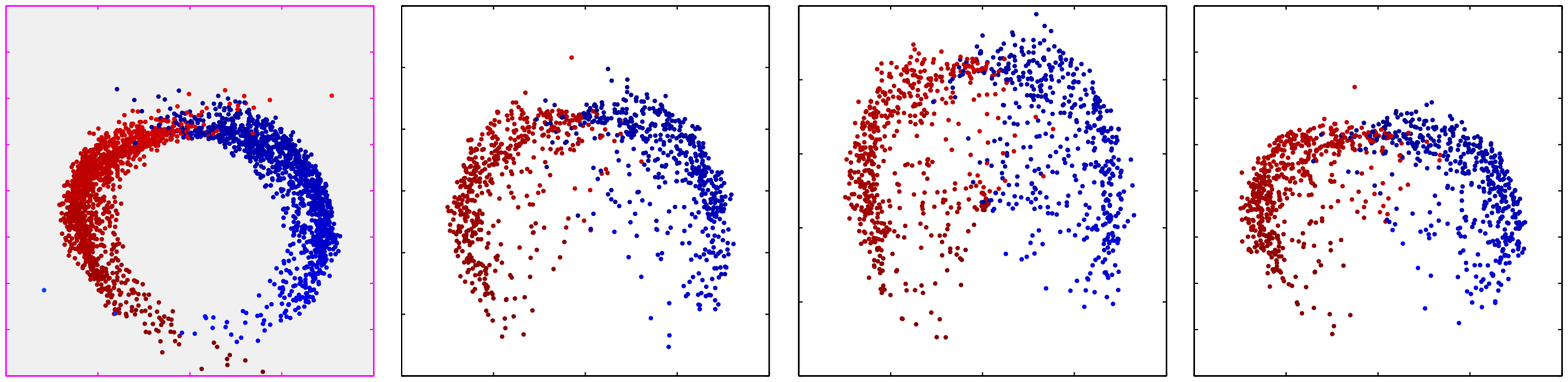} \\
    DICA
  \end{tabular} 
  \caption{Projections of a synthetic dataset onto the first two eigenvectors obtained from the KPCA, UDICA, COIR, and DICA. The colors of data points corresponds to the output values. The shaded boxes depict the projection of training data, whereas the unshaded boxes show projections of unseen test datasets. The feature representations learnt by UDICA and DICA are more stable across test domains than those learnt by KPCA and COIR.% In addition, DICA also preserves the functional relationship between $X$ and $Y$.
  }
  \label{fig:synthetic} 
\end{figure*} 
%%%%%%%%%%%%%%%%%%%%%%%%%%%%%%%%%%%%%%%%%%%%%%%%%%%%

\subsection{Gating of Flow Cytometry Data}
\label{sec:cytometry}

Graft-versus-host disease (GvHD) occurs in allogeneic hematopoietic stem cell transplant recipients when donor-immune cells in the graft recognize the recipient as ``foreign'' and initiate an attack on the skin, gut, liver, and other tissues. It is a significant clinical problem in the field of allogeneic blood and marrow transplantation. The GvHD dataset \citep{Brinkman07FCM} consists of weekly peripheral blood samples obtained from 31 patients following allogenic blood and marrow transplant. The goal of gating is to identify $\mbox{CD}3^+\mbox{CD}4^+\mbox{CD}8\beta^+$ cells, which were found to have a high correlation with the development of GvHD \citep{Brinkman07FCM}. We expect to find a subspace of cells that is consistent to the biological variation between patients, and is indicative of the GvHD development. For each patient, we select a dataset that contains sufficient numbers of the target cell populations. As a result, we omit one patient due to insufficient data. The corresponding flow cytometry datasets from 30 patients have sample sizes ranging from 1,000 to 10,000, and the proportion of the $\mbox{CD}3^+\mbox{CD}4^+\mbox{CD}8\beta^+$ cells in each dataset ranges from 10\% to 30\%, depending on the development of the GvHD. 

%%%%%%%%%%%%%%%%%%%%%%%%%%%%%%%%%%%%%%%%%%%%%%%%%%%%
%% the experimental result 
\begin{table*}[t]
  \centering
  \caption{Average accuracies over 30 random subsamples of GvHD datasets. Pooling SVM applies standard kernel function on the pooled data from multiple domains, whereas distributional SVM also considers similarity between domains using kernel \eqref{eq:dist-kernel}. With sufficiently many samples, DICA outperforms other methods in both pooling and distributional settings. The performance of pooling SVM and distributional SVM are comparable in this case.}
  \resizebox{\linewidth}{!}{
  \begin{tabular}{|l|lll|lll|}    
    \hline
    \multirow{2}{*}{\textbf{Methods}} & \multicolumn{3}{c}{\textbf{Pooling SVM}} & \multicolumn{3}{|c|}{\textbf{Distributional SVM}}\\
    & $n_i=100$ & $n_i=500$ & $n_i=1000$ & $n_i=100$ & $n_i=500$ & $n_i=1000$ \\
    \hline\hline
    Input & 91.68$\pm$.91 & 92.11$\pm$1.14 & 93.57$\pm$.77 & 91.53$\pm$.76 & 92.81$\pm$.93 & 92.41$\pm$.98 \\
    KPCA & 91.65$\pm$.93 & 92.06$\pm$1.15 & 93.59$\pm$.77 & \textbf{91.83$\pm$.60} & 90.86$\pm$1.98 & 92.61$\pm$1.12 \\
    COIR & \textbf{91.71$\pm$.88} & 92.00$\pm$1.05 & 92.57$\pm$.97 & 91.42$\pm$.95 & 91.54$\pm$1.14 & 92.61$\pm$.89 \\
    UDICA & 91.20$\pm$.81 & 92.21$\pm$.19 & 93.02$\pm$.77 & 91.51$\pm$.79 & 91.74$\pm$1.08 & 93.02$\pm$.77 \\
    DICA & 91.37$\pm$.91 & \textbf{92.71$\pm$.82} & \textbf{94.16$\pm$.73} & 91.51$\pm$.89 & \textbf{93.42$\pm$.73} & \textbf{93.33$\pm$.86} \\
    \hline 
  \end{tabular}}
\label{tab:flow-results}
\end{table*}
%%%%%%%%%%%%%%%%%%%%%%%%%%%%%%%%%%%%%%%%%%%%%%%%%%%%

To evaluate the performance of the proposed algorithms, we took data from $N=10$ patients for training, and the remaining 20 patients for testing. We subsample the training sets and test sets to have 100, 500, and 1,000 data points (cells) each. We compare the SVM classifiers under two settings, namely, a pooling SVM and a distributional SVM. The pooling SVM disregards the inter-patient variation by combining all datasets from different patients, whereas the distributional SVM also takes the inter-patient variation into account via the kernel function \citep{Blanchard:11Generalizing}
\begin{equation}
  \label{eq:dist-kernel}
  K(\tilde{x}^{(i)}_{k},\tilde{x}^{(j)}_{l}) = k_1(\pp{P}^i,\pp{P}^j)\cdot k_2(x_k^{(i)},x_l^{(j)})
\end{equation}
\noindent where $\tilde{x}^{(i)}_{k}=(\pp{P}^i,x_k^{(i)})$ and $k_1$ is the kernel on distributions. We use $k_1(\pp{P}^i,\pp{P}^j)=\exp\left(-\|\mu_{\pp{P}^i}-\mu_{\pp{P}^j}\|^2_{\hbspace}/2\sigma_1^2\right)$ and $k_2(x_k^{(i)},x_l^{(j)})=\exp(-\|x_k^{(i)}-x_l^{(j)}\|^2/2\sigma_2^2)$, where $\mu_{\pp{P}^i}$ is computed using $k_2$. For pooling SVM, the kernel $k_1(\pp{P}^i,\pp{P}^j)$ is constant for any $i$ and $j$. Moreover, we use the output kernel $l(y^{(i)}_k,y^{(j)}_l)=\delta(y^{(i)}_k,y^{(j)}_l)$ where $\delta(a,b)$ is 1 if $a=b$, and 0 otherwise. We compare the performance of the SVMs trained on the preprocessed datasets using the KPCA, COIR, UDICA, and DICA algorithms. It is important to note that we are not defining another kernel on top of the preprocessed data. That is, the kernel $k_2$ for KPCA, COIR, UDICA, and DICA is exactly \eqref{e:bkernel}. We perform 10-fold cross validation on the parameter grids to optimize for accuracy.

%%%%%%%%%%%%%%%%%%%%%%%%%%%%%%%%%%%%%%%%%%%%%%%%%%%%
\begin{table}[t!]
    \centering
    \caption{The average leave-one-out accuracies over 30 subjects on GvHD data. The distributional SVM outperforms the pooling SVM. DICA improves classifier accuracy.}
    \begin{tabular}{|lrr|}
      \hline
      \textbf{Methods} & \textbf{Pooling SVM} & \textbf{Distributional SVM} \\
      \hline\hline
      Input & 92.03$\pm$8.21 & 93.19$\pm$7.20 \\
      KPCA & 91.99$\pm$9.02 & 93.11$\pm$6.83 \\
      COIR & 92.40$\pm$8.63 & 92.92$\pm$8.20 \\
      UDICA & 92.51$\pm$5.09 & 92.74$\pm$5.01 \\
      DICA & \textbf{92.72$\pm$6.41} & \textbf{94.80$\pm$3.81} \\
      \hline
    \end{tabular}   
    \label{tab:flowloo}
  \end{table}
%%%%%%%%%%%%%%%%%%%%%%%%%%%%%%%%%%%%%%%%%%%%%%%%%%%%

Table \ref{tab:flow-results} reports average accuracies and their standard deviation over 30 repetitions of the experiments. For sufficiently large number of samples, DICA outperforms other approaches. The pooling SVM and distributional SVM achieve comparable accuracies. The average leave-one-out accuracies over 30 subjects are reported in Table \ref{tab:flowloo} (see supplementary for more detail).

\subsection{Parkinson's Telemonitoring}

To evaluate DICA in a regression setting, we apply it to a Parkinson's telemonitoring dataset\footnote{\href{http://archive.ics.uci.edu/ml/datasets/Parkinson's+Telemonitoring}{\texttt{http://archive.ics.uci.edu/ml/datasets/Parkinson's+Telemonitoring}}}. The dataset consists of biomedical voice measurements from 42 people with early-stage Parkinson's disease recruited for a six-month trial of a telemonitoring device for remote symptom progression monitoring. The aim is to predict the clinician's motor and total UPDRS scoring of Parkinson's disease symptoms from 16 voice measures. There are around 200 recordings per patient.

\begin{table*}[t]
  \centering
  \caption{Root mean square error (RMSE) of the independent Gaussian Process regression (GPR) applied to the Parkinson's telemonitoring dataset. DICA outperforms other approaches in both settings; and the distributional SVM outperforms the pooling SVM.}
  \begin{tabular}{|l|rr|rr|}
    \hline
    \multirow{2}{*}{\textbf{Methods}} & \multicolumn{2}{|c}{\textbf{Pooling GP Regression}} & \multicolumn{2}{|c|}{\textbf{Distributional GP Regression}} \\
    & motor score & total score & motor score & total score \\
    \hline
    LLS & 8.82 $\pm$ 0.77 & 11.80 $\pm$ 1.54 & 8.82 $\pm$ 0.77 & 11.80 $\pm$ 1.54 \\
    Input & 9.58 $\pm$ 1.06 & 12.67 $\pm$ 1.40 & 8.57 $\pm$ 0.77 & 11.50 $\pm$ 1.56 \\
    KPCA & 8.54 $\pm$ 0.89 & 11.20 $\pm$ 1.47 & 8.50 $\pm$ 0.87 & 11.22 $\pm$ 1.49\\
    UDICA & 8.67 $\pm$ 0.83 & 11.36 $\pm$ 1.43 & 8.75 $\pm$ 0.97 & 11.55 $\pm$ 1.52\\
    COIR & 9.25 $\pm$ 0.75 & 12.41 $\pm$ 1.63 & 9.23 $\pm$ 0.90 & 11.97 $\pm$ 2.09\\
    DICA & \textbf{8.40 $\pm$ 0.76} & \textbf{11.05 $\pm$ 1.50} & \textbf{8.35 $\pm$ 0.82} & \textbf{10.02 $\pm$ 1.01}\\
    \hline
  \end{tabular}
  \label{t:parkinson}
\end{table*}

We adopt the same experimental settings as in \S\ref{sec:cytometry}, except that we employ two independent Gaussian Process (GP) regression to predict motor and total UPDRS scores. For COIR and DICA, we consider the output kernel $l(y^{(i)}_k,y^{(j)}_l)=\exp(-\|y^{(i)}_k - y^{(j)}_l\|^2/2\sigma_3^2)$ to fully account for the affinity structure of the output variable. We set $\sigma_3$ to be the median of motor and total UPDRS scores. The voice measurements from 30 patients are used for training and the rest for testing. 

%% the experimental results on Parkinson's telemonitoring
\begin{figure}[t!] 
  \centering
  \includegraphics[height=4in,angle=-90]{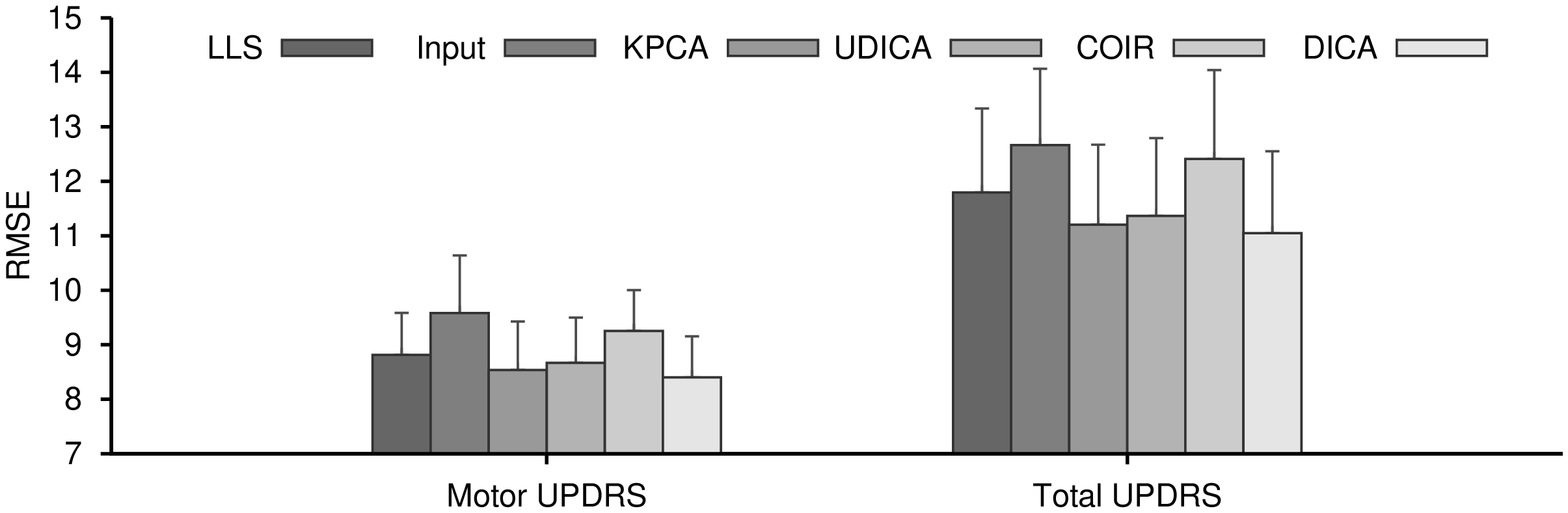}
  \includegraphics[height=4in,angle=-90]{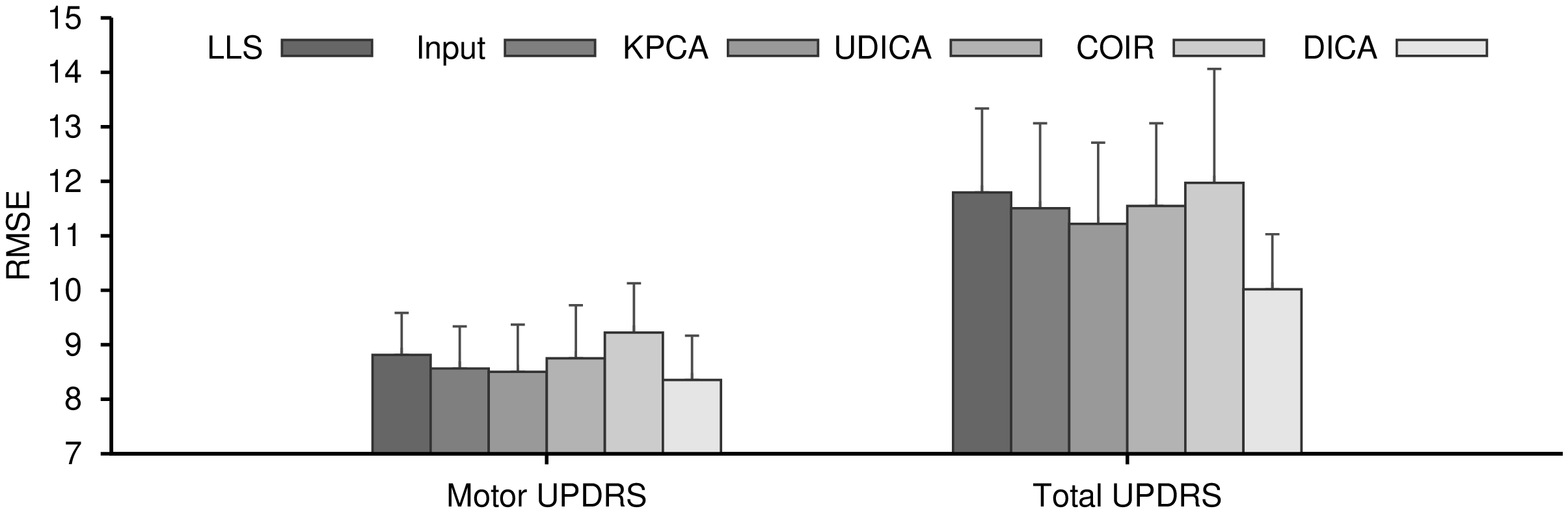}
  \includegraphics[height=4in,angle=-90]{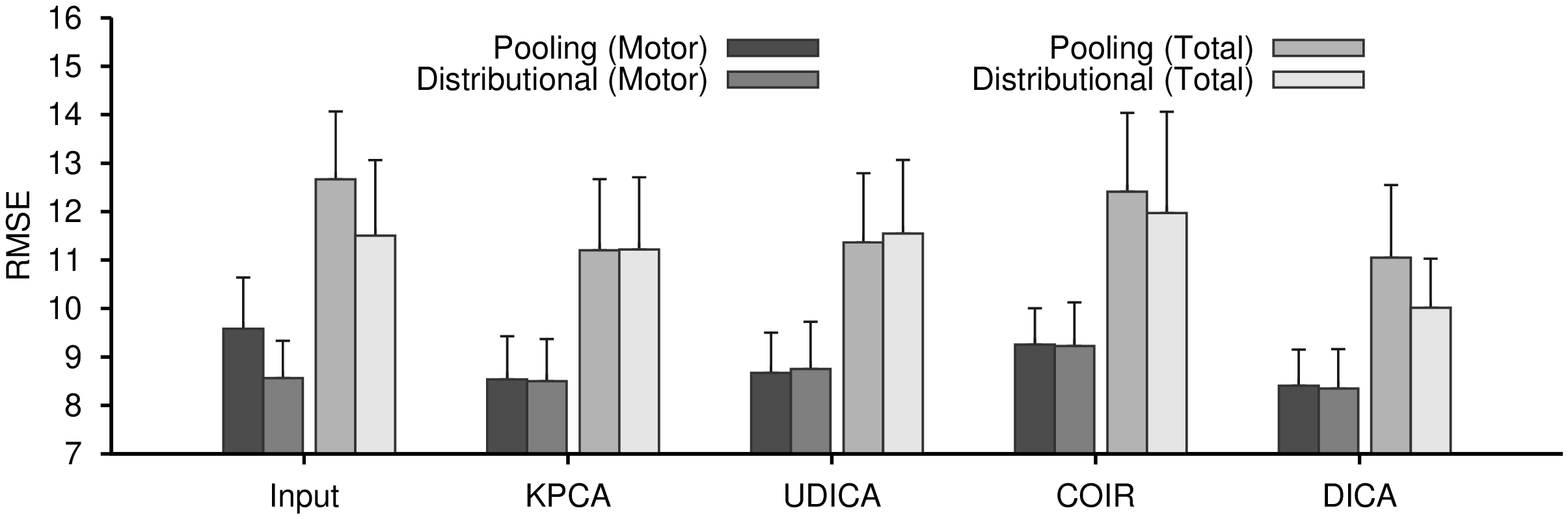}
  \caption{The root mean square error (RMSE) of motor and total UPDRS scores predicted by GP regression after different preprocessing methods on Parkinson's telemonitoring dataset. The top and middle rows depicts the pooling and distributional settings; the bottom row compares the two settings. Results of linear least square (LLS) are given as a baseline.}
  \label{fig:Parkinsons}
\end{figure}

Fig.~\ref{fig:Parkinsons} depicts the results. DICA consistently, though not statistically significantly, outperforms other approaches, see Table~\ref{t:parkinson}.  Inter-patient (i.e. across domain) variation worsens prediction accuracy on new patients. Reducing this variation with DICA improves the accuracy on new patients. Moreover, incorporating the inter-subject variation via distributional GP regression further improves the generalization ability, see Fig.~\ref{fig:Parkinsons}.% Finally, using DICA in a distributional setting tends to best reduce the RMSE for the total UPDRS score.

%% discussions
\section{Conclusion and Discussion} 

%Although DICA can be computed analytically, matrix inversion and eigen-decomposition on large-scale datasets is expensive. The former can be accelerated by assuming a low-rank structure for kernel matrix $\mathrm{K}=\Gamma\Gamma^{\top}$. The latter can be efficiently solved iteratively.

To conclude, we proposed a simple algorithm called Domain-Invariant Component Analysis (DICA) for learning an invariant transformation of the data which has proven significant for domain generalization both theoretically and empirically. Theorem~\ref{t:lbound} shows the generalization error on previously unseen domains grows with the distributional variance. We also showed that DICA generalizes KPCA and COIR, and is closely related to TCA. Finally, experimental results on both synthetic and real-world datasets show DICA performs well in practice. Interestingly, the results also suggest that the distributional SVM, which takes into account inter-domain variation, outperforms the pooling SVM which ignores it.

The motivating assumption in this work is that the functional relationship is stable or varies smoothly across domains. This is a reasonable assumption for automatic gating of flow cytometry data because the inter-subject variation of cell population makes it impossible for domain expert to apply the same gating on all subjects, and similarly makes sense for Parkinson's telemonitoring data. 
Nevertheless, the assumption does not hold in many applications where the conditional distributions are substantially different. It remains unclear how to develop techniques that generalize to previously unseen domains in these scenarios. %This line of research has been studied in multitask learning when one aims to apply a knowledge learnt in the training tasks to a previously unseen task. Most traditional approaches require re-training on the previously unseen tasks which thereby renders them seemingly inappropriate for many real-world applications. 

DICA can be adapted to novel applications by equipping the optimization problem with appropriate constraints. For example, one can formulate a semi-supervised extension of DICA by forcing the invariant basis functions to lie on a manifold or preserve a neighborhood structure. Moreover, by incorporating the distributional variance as a regularizer in the objective function, the invariant features and classifier can be optimized simultaneously. 

%% acknowledgments
\subsubsection*{Acknowledgments}
We thank Samory Kpotufe and Kun Zhang for fruitful discussions and the three anonymous reviewers for insightful comments and suggestions that significantly improved the paper. 

%\newpage
\bibliographystyle{abbrvnat}
\bibliography{dica}

%%%%%%%%%%%%%%%%%%%%%%%%%%%%%%%%%%%%%%%%%%%%%%%%%%%%%%%%%%%
\newpage
\appendix

\section{Domain Generalization and Related Frameworks}

The most fundamental assumption in machine learning is that the observations are independent and identically distributed (i.i.d.). That is, each observation comes from the same probability distribution as the others and all are mutually independent. However, this assumption is often violated in practice, in which case the standard machine learning algorithms do not perform well. In the past decades, many techniques have been proposed to tackle scenarios where there is a mismatch between training and test distributions. These include domain adaptation \citep{Bickel09:DLU}, multitask learning \citep{caruana:97}, transfer learning \citep{Pan10:SurveyTF}, covariate/dataset shift \citep{quionero:09} and concept drift \citep{widmer:96}. To better understand domain generalization, we briefly discuss how it relates to some of these approaches.

%% transfer learning
\subsection{Transfer learning (see e.g., \citet{Pan10:SurveyTF} and references therein).}
Transfer learning aims at transferring knowledge from some previous tasks to a target task when the latter has limited training data. That is, although there may be few labeled examples, ``knowledge'' obtained in related tasks may be available. Transfer learning focuses on improving the learning of the target predictive function using the knowledge in the source task. Although not identical, domain generalization can be viewed as a transfer learning when knowledge of the target task is unavailable during training.

%% multitask learning
\subsection{Multitask learning (see e.g., \citet{caruana:97} and references therein).}
  The goal of multitask learning is to learn multiple tasks simultaneously -- especially when training examples in each task are scarce. By learning all tasks simultaneously, one expects to improve generalization on  individual tasks. An important assumption is therefore that all the tasks are related. Multitask learning differs from domain generalization because learning the new task often requires retraining.

%% domain adaptation
\subsection{Domain adaptation (see e.g., \citet{Bickel09:DLU} and references therein).}
  Domain adaptation, also known as covariate shift, deals primarily with a mismatch between training and test distributions. Domain generalization deals with a broader setting  where training instances may have been collected from multiple source domains. A second difference is that in domain adaptation one observes the target domain during the training time whereas in domain generalization one does not.

Table \ref{tab:diff-summary} summarizes the main differences between the various frameworks.% Due to an improving technological advance, many emerging scientific disciplines can easily produce a tremendous amount of data. In such cases, the traditional assumptions are often violated. Therefore, we expect the domain generalization framework to become increasingly important in machine learning community.

%%% table summarizing the difference
\begin{table*}[t!]
  \centering
  \caption{Comparison of domain generalization with other well-known frameworks. Note that the domain generalization is closely related to multi-task learning and domain adaptation. The difference of domain generalization is that one does not observe the target domains in which a classifier will be applied without retraining the classifier.}
  \resizebox{\linewidth}{!}{
  \begin{tabular}{|lccc|}
    \hline
    Framework & Distribution Mismatch & Multiple Sources & Target Domain \\
    \hline 
    Standard Setup & \cross & \cross & \cross \\
    Transfer Learning & \tick & \cross & \tick \\
    Multi-task Learning & \tick & \tick & \cross \\
    Domain Adaptation & \tick & \tick & \tick \\
    Domain Generalization & \tick & \tick & \cross \\
    \hline
  \end{tabular}}
  \label{tab:diff-summary}
\end{table*}

\section{Proof of Theorem \ref{thm:uniqueness}}
 
\begin{mylem}
  \label{lem:uniqueness}
  Given a set of distributions $\mathcal{P}=\{\pp{P}^1,\pp{P}^2\dotsc,\pp{P}^N\}$, the distributional variance of $\mathcal{P}$ is $\mathbb{V}_{\hbspace}(\mathcal{P}) = \frac{1}{N}\sum_{i=1}^N\norm{\abbrvmm{\pp{P}^i} - \abbrvmm{\bar{\pp{P}}}}_{\hbspace}^2$ where $\abbrvmm{\bar{\pp{P}}}=(1/N)\sum_{i=1}^N\abbrvmm{\pp{P}^i}$ and $\bar{\pp{P}}=\frac{1}{N}\sum_{i=1}^N\pp{P}^i$.
\end{mylem}

\begin{proof}
  Let $\bar{\pp{P}}$ be the probability distribution defined as $(1/N)\sum_{i=1}^N\pp{P}^i$, i.e., $\bar{\pp{P}}(x) = (1/N)\sum_{i=1}^N\pp{P}^i(x)$. It follows from the linearity of the expectation that $\abbrvmm{\bar{\pp{P}}}=(1/N)\sum_{i=1}^N\abbrvmm{\pp{P}^i}$. For brevity, we will denote $\langle\cdot,\cdot\rangle_{\hbspace}$ by $\langle\cdot,\cdot\rangle$. Then, expanding \eqref{eq:tr-variance} gives
  \begin{eqnarray*}    
    \mathbb{V}_{\hbspace}(\mathcal{P}) 
    &=& \frac{1}{N}\mathrm{tr}(\Sigma) = \frac{1}{N}\text{tr}(G) - \dfrac{1}{N^2}\sum_{i,j=1}^{N}G_{ij} \\
    &=& \frac{1}{N}\sum_{i=1}^{N}\langle\abbrvmm{\pp{P}^i},\abbrvmm{\pp{P}^i}\rangle - \frac{1}{N^2}\sum_{i,j=1}^{N}\langle\abbrvmm{\pp{P}^i},\abbrvmm{\pp{P}^j}\rangle \\
    &=& \frac{1}{N}\left[\sum_{i=1}^{N}\langle\abbrvmm{\pp{P}^i},\abbrvmm{\pp{P}^i}\rangle
    - \frac{2}{N}\sum_{i,j=1}^{N}\langle\abbrvmm{\pp{P}^i},\abbrvmm{\pp{P}^j}\rangle + \frac{1}{N}\sum_{i,j=1}^{N}\langle\abbrvmm{\pp{P}^i},\abbrvmm{\pp{P}^j}\rangle\right] \\
    &=& \frac{1}{N}\left[\sum_{i=1}^{N}\langle\abbrvmm{\pp{P}^i},\abbrvmm{\pp{P}^i}\rangle 
    - 2\sum_{i=1}^{N}\left\langle \abbrvmm{\pp{P}^i},\frac{1}{N}\sum_{j=1}^{N}
      \abbrvmm{\pp{P}^j} \right\rangle + N \left\langle\frac{1}{N}\sum_{i=1}^{N}\abbrvmm{\pp{P}^i},\frac{1}{N} \sum_{j=1}^{N}\abbrvmm{\pp{P}^j}\right\rangle \right]\\
    &=& \frac{1}{N}\left[\sum_{i=1}^{N}\langle\abbrvmm{\pp{P}^i},\abbrvmm{\pp{P}^i}\rangle
    - 2\sum_{i=1}^{N}\langle \abbrvmm{\pp{P}^i},\abbrvmm{\bar{\pp{P}}}\rangle
    + N \langle \abbrvmm{\bar{\pp{P}}},\abbrvmm{\bar{\pp{P}}}\rangle \right] \\
    &=& \frac{1}{N}\sum_{i=1}^{N}\Big( \langle\abbrvmm{\pp{P}^i},\abbrvmm{\pp{P}^i}\rangle-2\cdot\langle \abbrvmm{\pp{P}^i},\abbrvmm{\bar{\pp{P}}}\rangle 
    + \langle \abbrvmm{\bar{\pp{P}}},\abbrvmm{\bar{\pp{P}}}\rangle \Big) \\
    &=& \frac{1}{N}\sum_{i=1}^{N}\norm{\abbrvmm{\pp{P}^i} - \abbrvmm{\bar{\pp{P}}}}_{\hbspace}^2 \enspace ,
  \end{eqnarray*}
  \noindent which completes the proof.
\end{proof}

\begin{thmtag}{\ref{thm:uniqueness}}
  \emph{For a characteristic kernel $k$, $\mathbb{V}_{\hbspace}(\mathcal{P})=0$  if and only if $\pp{P}^1 =\pp{P}^2=\cdots=\pp{P}^{N}$.}
\end{thmtag}

\begin{proof}
  Since $k$ is characteristic, $\norm{\abbrvmm{\pp{P}}-\abbrvmm{\pp{Q}}}^2_{\hbspace}$ is a metric and is zero iff $\pp{P}=\pp{Q}$ for any distributions $\pp{P}$ and $\pp{Q}$ \citep{Sriperumbudur10:Metrics}. By Lemma~\ref{lem:uniqueness}, $\mathbb{V}_{\hbspace}(\mathcal{P}) = \frac{1}{N}\sum_{i=1}^N\norm{\abbrvmm{\pp{P}^i} - \abbrvmm{\bar{\pp{P}}}}_{\hbspace}^2.$ Thus, $\norm{\abbrvmm{\pp{P}^i} - \abbrvmm{\bar{\pp{P}}}}_{\hbspace}^2=0$ iff $\pp{P}^i
  =\bar{\pp{P}}$. Consequently, if $\mathbb{V}_{\hbspace}(\mathcal{P})$ is zero, this implies that
  $\pp{P}^i  =\bar{\pp{P}}$ for all $i$, meaning that $\pp{P}^1=\cdots=\pp{P}^{\ell}$.
  Conversely, if $\pp{P}^1=\cdots=\pp{P}^{\ell}$, then $\norm{\abbrvmm{\pp{P}^i} - \abbrvmm{\bar{\pp{P}}}}_{\hbspace}^2=0$ is zero for all $i$ and thereby $\mathbb{V}_{\hbspace}(\mathcal{P})=\frac{1}{N}\sum_{i=1}^N\norm{\abbrvmm{\pp{P}^i} - \abbrvmm{\bar{\pp{P}}}}_{\hbspace}^2$ is zero. 
\end{proof}
%% proof

\section{Proof of Theorem \ref{thm:vestimator}}

\begin{thmtag}{\ref{thm:vestimator}}
 \emph{ The empirical estimator $\widehat{\mathbb{V}}_{\hbspace}(\sset) = \frac{1}{N}\mathrm{tr}(\widehat{\Sigma}) = \mathrm{tr}(\kxmat\qmat)$ obtained from Gram matrix
   \begin{equation*}
     %\label{eq:empirical-kernel}
     \widehat{\gmat}_{ij} 
    :=\dfrac{1}{n_i\cdot n_j}\sum_{k=1}^{n_i}\sum_{l=1}^{n_j}
     k(x^{(i)}_k,x^{(j)}_l)
   \end{equation*}
   is a consistent estimator of $\mathbb{V}_{\hbspace}({\mathcal P})$.}
\end{thmtag}

\begin{proof}
Recall that
\begin{equation*}
  \mathbb{V}_{\hbspace}(\mathcal{P}) = \frac{1}{N}\mathrm{tr}(G) - \frac{1}{N^2}\sum_{i,j=1}^N G_{ij} \enspace \text{ and } \enspace \widehat{\mathbb{V}}_{\hbspace}(\mathcal{S}) = \frac{1}{N}\mathrm{tr}(\widehat{G}) - \frac{1}{N^2}\sum_{i,j=1}^N \widehat{G}_{ij}
\end{equation*}
\noindent where 
\begin{eqnarray*}
  G_{ij}&=& \langle \mu_{\pp{P}^i},\mu_{\pp{P}^j}\rangle_{\hbspace}= \iint k(x,z) \dd\pp{P}^i(x)\dd\pp{P}^j(z)\\
  \widehat{G}_{ij}&=& \langle \hat{\mu}_{\pp{P}^i},\hat{\mu}_{\pp{P}^j}\rangle_{\hbspace} = \frac{1}{n_in_j}\sum_{k=1}^{n_i}\sum_{l=1}^{n_j}k(x_k^{(i)},x_l^{(j)}) 
\end{eqnarray*}
By Theorem 15 in \citet{Altun06:unifyingdivergence}, we have a fast convergence of $\hat{\mu}_{\pp{P}}$ to $\mu_{\pp{P}}$. Consequently, we have $\widehat{G}\rightarrow G$, which implies that $\widehat{\mathbb{V}}_{\hbspace}(\mathcal{S}) \rightarrow\mathbb{V}_{\hbspace}(\mathcal{P})$. Hence, $\widehat{\mathbb{V}}_{\hbspace}(\mathcal{S})$ is a consistent estimator of $\mathbb{V}_{\hbspace}(\mathcal{P})$.
\end{proof}

\section{Derivation of Eq. \eqref{eq:flda}}

DICA employs the covariance of inverse regressor $\mathbb{V}(\mathbb{E}[\phi(X)|Y])$, which can be written in terms of covariance operators. Let $\hbspace$ and $\hbspacey$ be the RKHSes of $X$ and $Y$ endowed with reproducing kernels $k$ and $l$, respectively. Let $\covxx$, $\covyy$, $\covxy$, and $\covyx$ be the covariance operators in and between the corresponding RKHSes of $X$ and $Y$. We define the conditional covariance operator of $X$ given $Y$, denoted by $\Sigma_{\mathrm{xx}|\mathrm{y}}$, as%\footnote{This is an abuse of notation because $\covyy^{-1}$ may not exist. However, for convenience, we will use this notation throughout the paper.}
\begin{equation}
  \label{eq:condcov-optr}
  \Sigma_{\mathrm{xx}|\mathrm{y}} \triangleq \covxx - \covxy\covyy^{-1}\covyx \enspace .
\end{equation}

The following theorem from \citet{Fukumizu04:RKHS} states that, under mild conditions, $\Sigma_{\mathrm{xx}|\mathrm{y}}$ equals the expected conditional variance of $\phi(X)$ given $Y$.
%%% theorem
\begin{mythm}
  \label{thm:cov-optr}
  For any $f\in\hbspace$, if there exists $g\in\hbspacey$ such that $\mathbb{E}[f(X)|Y]=g(Y)$ for almost every $Y$, then $\Sigma_{\mathrm{xx}|\mathrm{y}} = \mathbb{E}[\mathbb{V}(\phi(X)|Y)]$.
\end{mythm}

Using the $E\text{-}V\text{-}V\text{-}E$ identity\footnote{$\mathbb{V}(X) = \mathbb{E}[\mathbb{V}(X|Y)] + \mathbb{V}(\mathbb{E}[X|Y])$ for any $X,Y$.}, the covariance $\mathbb{V}(\mathbb{E}[\phi(X)|Y])$ can be expressed in terms of the conditional covariance operators as follow:
\begin{equation}
  \label{eq:evve}
  \mathbb{V}(\mathbb{E}[\phi(X)|Y]) = \mathbb{V}(\phi(X)) - \mathbb{E}[\mathbb{V}(\phi(X)|Y)],
\end{equation}
assuming that the inverse regressor $\mathbb{E}[f(x)|y]$ is a smooth function of $y$ for any $f\in\hbspace$. 

By virtue of Theorem \ref{thm:cov-optr}, the second term in the r.h.s. of \eqref{eq:evve} is $\Sigma_{\mathrm{xx}|\mathrm{y}}$. Since $\mathbb{V}(\phi(X)) = \mathrm{Cov}(\phi(x),\phi(x))=\covxx$, it follows from \eqref{eq:condcov-optr} that the covariance of the inverse regression $\mathbb{V}(\mathbb{E}[\phi(X)]|Y)$ can be expressed as
\begin{equation}
  \label{eq:covariance}
  \mathbb{V}(\mathbb{E}[\phi(X)|Y]) = \covxy\covyy^{-1}\covyx \enspace .
\end{equation}

The covariance \eqref{eq:covariance} can be estimated from finite samples $(x_1,y_1),\dotsc,(x_n,y_n)$ by $\widehat{\mathbb{V}}(\mathbb{E}[\phi(X)|Y]) = \widehat{\Sigma}_{\mathrm{xy}}\widehat{\Sigma}_{\mathrm{yy}}^{-1}\widehat{\Sigma}_{\mathrm{yx}}$ where $\widehat{\Sigma}_{\mathrm{xy}} = \frac{1}{n}\Phi_x\Phi_y^{\top}$ and $\Phi_x = [\phi(x_1),\dotsc,\phi(x_n)]$ and $\Phi_y = [\varphi(y_1),\dotsc,\varphi(y_n)]$. Let $\kxmat$ and $\kymat$ denote the kernel matrices computed over samples $\{x_1,x_2,\dotsc,x_n\}$ and $\{y_1,y_2,\dotsc,y_n\}$, respectively. We have
\begin{flalign}
   \widehat{\mathbb{V}}(\mathbb{E}[\phi(X)|Y])  
  & = \left(\frac{1}{n}\Phi_x\Phi_y^{\top}\right)\left(\frac{1}{n}(\Phi_y\Phi_y^{\top}
    + n\varepsilon \mathcal{I}) \right)^{-1}\left(\frac{1}{n}\Phi_y\Phi_x^{\top}\right) \nonumber \\
  & = \frac{1}{n}\Phi_x\Phi_y^{\top}\Phi_y\left(\Phi_y^{\top}\Phi_y 
    + n\varepsilon I_n\right)^{-1}\Phi_x^{\top} \nonumber \\
  & = \frac{1}{n}\Phi_x\kymat\left(\kymat + n\varepsilon I_n\right)^{-1}\Phi_x^{\top} 
  \label{eq:cov-ir} 
\end{flalign}
\noindent where $\kymat=\Phi_y^{\top}\Phi_y$ and $\mathcal{I}$ is the identity operator. The second equation is obtained by applying the fact that $(\Phi_y\Phi_y^{\top} + n\varepsilon\mathcal{I})\Phi_y = \Phi_y(\Phi_y^{\top}\Phi_y + n\varepsilon\id_n)$. 

Finally, using $\widehat{\Sigma}_{\mathrm{xx}}=\frac{1}{n}\Phi_x\Phi_x^{\top}$ and recalling that $\kxmat=\Phi_x^{\top}\Phi_x$, we obtain
\begin{align*}
	\bvec_k^{\top}\widehat{\Sigma}_{\mathrm{xx}}^{-1}\widehat{\mathbb{V}}(\ep[X|Y])\widehat{\Sigma}_{\mathrm{xx}}\bvec_k 
	& = \bvec_k^{\top}\left(\frac{1}{n}\Phi_x\Phi_x^{\top}\right)^{-1}\left(\frac{1}{n}\Phi_x\kymat\left(\kymat + n\varepsilon 
	I_n\right)^{-1}\Phi_x^{\top}\right)\left(\frac{1}{n}\Phi_x\Phi_x^{\top}\right)\bvec_k \\
	& = \frac{1}{n}\bm{\beta}_k^{\top}\Phi_x^{\top}\left(\Phi_x\Phi_x^{\top}\right)^{-1}\Phi_x\kymat\left(\kymat + n\varepsilon
	I_n\right)^{-1}\Phi_x^{\top}\left(\Phi_x\Phi_x^{\top}\right)\Phi_x\bm{\beta}_k \\
      & = \frac{1}{n}\bm{\beta}_k^{\top}\Phi_x^{\top}\Phi_x\left(\Phi_x^{\top}\Phi_x\right)^{-1}\kymat\left(\kymat + n\varepsilon
	I_n\right)^{-1}\Phi_x^{\top}\left(\Phi_x\Phi_x^{\top}\right)\Phi_x\bm{\beta}_k \\
      & = \frac{1}{n}\bm{\beta}_k^{\top}\kymat(\kymat+n\varepsilon\id)^{-1}\kxmat^2\bm{\beta}_k
  \end{align*}
and 
  \begin{equation*}
  \bvec_k^{\top}\bvec_k 
  =\bm{\beta}_k^{\top} \Phi_x^{\top}\Phi_x\bm{\beta}_k 
  = \bm{\beta}_k^{\top}\kxmat\bm{\beta}_k 
\end{equation*}
\noindent as desired.

%%%%%%%%%%%%%%%%%%%%%%%%%%%%%%%%%%%%%%%%%%%%%%%%%%%%%%%%%%%%%%%%%%%%%%
\section{Derivation of Lagrangian \eqref{eq:ssca-lagragian}}

Observe that optimization 
\begin{equation}
	\underset{\bmat\in\rr^{n\times m}}{\text{max}}\,
	\frac{\mathrm{tr}\left(\bmat^\top X\bmat\right)}{\mathrm{tr}\left(\bmat^\top Y\bmat\right)}
	\label{e:ratio}
\end{equation}
is invariant to rescaling $\bmat \mapsto \alpha\cdot \bmat$. Optimization \eqref{e:ratio} is therefore equivalent to
\begin{align*}
	\underset{\bmat\in\rr^{n\times m}}{\text{max}}\, &
	\mathrm{tr}\left(\bmat^\top X\bmat\right)\\
	\text{subject to: } & \mathrm{tr}\left(\bmat^\top Y\bmat\right) = 1,
\end{align*}
which yields Lagrangian
\begin{equation}
    \mathcal{L} = 
	\mathrm{tr}\left(\bmat^{\top}X \bmat\right) 
    - \mathrm{tr}\left(\left(\bmat^{\top}Y\bmat - \id\right)\Gamma\right).
\end{equation}

%%%%%%%%%%%%%%%%%%%%%%%%%%%%%%%%%%%%%%%%%%%%%%%%%%%%%%%%%%%%%%%%%%%%%%
\section{Proof of Theorem~\ref{t:lbound}}

We consider a scenario where distributions $\pp{P}^i$ are drawn according to $\bprob$ with probability $\mu_i$. Introduce shorthand $\tilde{X}_{ij}$ for $(\pp{P}^{(i)},X_{ij})$ for a distribution on $\pspace_\inspace$ and a corresponding random variable on $\inspace$. 

The quantity of interest is the difference between the expected and empirical loss of a classifier $f:\pspace_\inspace\times \inspace\rightarrow \outspace$ under loss function $\ell:\outspace\times\outspace\rightarrow \rr_+$.

\textbf{Assumptions.}
The loss function $\ell:\rr\times\outspace\rightarrow \rr_+$ is $\phi_{\ell}$-Lipschitz in its first variable and bounded by $U_{\ell}$. The kernel $k_\inspace$ is bounded by $U_\inspace$. Assume that all distributions in $\bprob$ are mapped into a ball of size $U_\pspace$ by $\Psi_\pspace$. Finally, since $k_\pspace$ is a is a square exponential, there is a constant $L_\pspace$ such that 
\begin{equation*}
	\|\Phi_\pspace(v)-\Phi_\pspace(w)\|\leq L_\pspace\|v-w\| \text{ for all }v,w.
\end{equation*}

Recall that $N$ is the number of sampled domains, $n_i$ is the number of samples in domain $i$, and $n=\sum_{i=1}^N n_i$ is the total number of samples. The proof assumes $n_i=n_j$ for all $i,j$.

\begin{thmtag}{\ref{t:lbound}.}{}
\emph{Assumes the conditions above hold. Then with probability at least $1-\delta$
	\begin{gather*}
          \sup_{\|f\|_{\hbspace}\leq 1} 
          \left|\ep_\bprob \ep_\pp{P}\ell(f(\tilde{X}_{ij}\bopt),Y_i) 
            -\ep_{\hat{\pp{P}}} \ell(f(\tilde{X}_{ij}\bopt),Y_i)\right|^2\\
          \leq  c_1 \frac{1}{N} \text{tr}(\bmat^\intercal \kxmat\lmat\kxmat\bmat)
          + \text{tr}(\bmat^\top\kxmat\bmat)\left(c_2\frac{N\cdot(\log\delta^{-1}+2\log N)}{n}
            + c_3\frac{\log\delta^{-1}}{N}
            + \frac{c_4}{N}\right).
	\end{gather*}
      }
\end{thmtag}
\begin{rem}
	Recall that $\Phi_x=[\phi(x_1),\ldots,\phi(x_n)]$. The composition $x_t\mapsto \mathbf{k}_t\cdot \bmat$, where $\mathbf{k}_t=[k(x_1,x_t),\ldots,k(x_n,x_t)]$, can therefore be rewritten as $\phi(x_t)\cdot\bopt=\phi(x_t)\cdot\Phi_x\cdot\bmat$.	
\end{rem}

\begin{proof}
	The proof modifies the approach taken in \citet{Blanchard:11Generalizing} to handle the preprocessing via transform $\bopt$, and the fact that we work with \emph{squared} errors. Parts of the proof that pass through largely unchanged are omitted.
		
	We repeatedly apply the inequality $|a+b|^2\leq 2|a|^2+2|b|^2$. However, we only incur the multiplication-by-2 penalty once since $|a_1+\cdots+a_n|^2\leq 2|a_1|^2+\cdots+2|a_n|^2$. 
	
	Decompose
\begin{align*}
 &\sup_{\|f\|_{\hbspace}\leq 1}
	\left|\ep_\bprob \ep_\pp{P}\ell(f(\tilde{X}_{ij}\bopt),Y_i) 
	-\ep_{\hat{\pp{P}}} \ell(f(\tilde{X}_{ij}\bopt),Y_i)\right|^2 & \\
      & \leq \sup_{\|f\|_{\hbspace}\leq 1}\frac{2}{N}\sum_{i=1}^{N}
	\left|\ep_\bprob \ep_\pp{P}\ell(f(\tilde{X}_{ij}\bopt),Y_i) 
	-\ep_{\pp{P}^i} \ell(f(\tilde{X}_{ij}\bopt),Y_i)\right|^2 & \\
      & +\sup_{\|f\|_{\hbspace}\leq 1}\frac{2}{N}\sum_{i=1}^{N}
	\left|\ep_{\pp{P}^i}\ell(f(\tilde{X}_{ij}\bopt),Y_i) 
	-\ep_{\widehat{\pp{P}}^i} \ell(f(\tilde{X}_{ij}\bopt),Y_i)\right|^2 & \\
      & +\sup_{\|f\|_{\hbspace}\leq 1}\frac{2}{N}\sum_{i=1}^{N}
	\left|\ep_{\widehat{\pp{P}}^i}\ell(f(\tilde{X}_{ij}\bopt),Y_i) 
	-\ep_{\hat{\pp{P}}} \ell(f(\tilde{X}_{ij}\bopt),Y_i)\right|^2 & \\	
      &= (A) + (B) + (C) \enspace .
\end{align*}

\textbf{Control of (C):}
\begin{align*}
	(C) & = \sup_{\|f\|_{\hbspace}\leq 1}\frac{2}{N}\sum_{i=1}^{N}
	\left|\ep_{\widehat{\pp{P}}^i}\ell(f(\tilde{X}_{ij}\bopt),Y_i) 
	-\ep_{\hat{\pp{P}}} \ell(f(\tilde{X}_{ij}\bopt),Y_i)\right|^2\\
	& \leq \phi_{\ell}^2\sup_{\|f\|_{\hbspace}\leq 1} \frac{2}{N}\sum_{i=1}^{N} 
	\left|\ep_{\widehat{\pp{P}}^i}f(\tilde{X}_{ij}\bopt)
	- \ep_{\hat{\pp{P}}} f(\tilde{X}_{ij}\bopt)\right|^2 \\
	& = \phi_{\ell}^2\cdot \frac{2}{N}\sum_{i=1}^{N} 
	\left\|\Psi_\pspace(\widehat{\pp{P}^i})\otimes\mu_{\widehat{\pp{P}}^i} \bopt
	- \Psi_\pspace(\widehat{\pp{P}})\otimes\mu_{\widehat{\pp{P}}}\bopt\right\|^2
\end{align*}
	Note that $\|\Psi_\pspace(\mu(\pp{P}))\|^2\leq L_\pspace\cdot \| \mu_{\pp{P}}\|^2\leq L_\pspace U_\pspace$. Therefore,
\begin{align*}
	(C) & \leq \phi_{\ell}^2  L_\pspace U_\pspace \frac{2}{N}\sum_{i=1}^{N}
	\left\|\mu_{\widehat{\pp{P}}^i}\bopt - \mu_{\widehat{\pp{P}}}\bopt\right\|^2.
\end{align*}
By the proof of Theorem~\ref{thm:uniqueness} and since $\Phi^\top_x\bopt=\kxmat\bmat$, we have
\begin{align*}
	(C) & \leq 2\phi_{\ell}^2 L_\pspace U_\pspace\frac{1}{N} \text{tr}(\kxmat\bmat\bmat^\intercal \kxmat\lmat).
\end{align*}

\textbf{Control of (B):}
Similarly,
\begin{align*}
	(B) & = \sup_{\|f\|_{\hbspace}\leq 1}\frac{2}{N}\sum_{i=1}^{N}
	\left|\ep_{\pp{P}^i}\ell(f(\tilde{X}_{ij}\bopt),Y_i) 
	-\ep_{\widehat{\pp{P}}^i} \ell(f(\tilde{X}_{ij}\bopt),Y_i)\right|^2\\
	& \leq 2\phi_{\ell}^2 L_\pspace U_\pspace\cdot\frac{1}{N}\sum_{i=1}^{N} 
	\left\|\mu_{\pp{P}^i}\bopt - \mu_{\widehat{\pp{P}}^i}\bopt\right\|^2 \\
	& \leq 2\phi_{\ell}^2 L_\pspace U_\pspace\cdot \|\bopt\|^2_{HS}\cdot 
	\frac{1}{N}\sum_{i=1}^{N}\left\|\mu_{\pp{P}^i} - \mu_{\widehat{\pp{P}}^i}\right\|^2
\end{align*}

Here we follow the strategy applied by \citet{Blanchard:11Generalizing} to control their term (I) in Theorem 5.1. Assume $n_i=n_j$ for all $i,j$ and recall $n=\sum_{i=1}^N n_i$ so $n_i=n/N$ for all $i$. 
	
	By Hoeffding's inequality in Hilbert space, with probability greater than $1-\delta$ the following inequality holds
	\begin{equation*}
		\left\|\frac{1}{n_i}\sum_{j=1}^{n_i} \mu(\hat{X}_{ij})
		-  \ep_{\pp{P}^{(i)}}\mu( X_{ij})\right\|^2
		\leq 9U_\inspace\frac{N\cdot\log2\delta^{-1}}{n}.
	\end{equation*}
	Applying the union bound obtains
	\begin{equation*}
		(Ib) \leq 18\phi_{\ell}^2 L_\pspace U_\pspace U_\inspace\cdot\|\bopt\|^2_{HS}\cdot
		\frac{N\cdot(\log\delta^{-1}+2\log N)}{n}.
	\end{equation*}

\textbf{Control of (A):}	
\begin{align*}
	(A) & = \sup_{\|f\|_{\hbspace}\leq 1}\frac{2}{N}\sum_{i=1}^{N}
	\left|\ep_\bprob \ep_\pp{P}\ell(f(\tilde{X}_{ij}\bopt),Y_i) 
	-\ep_{\pp{P}^i} \ell(f(\tilde{X}_{ij}\bopt),Y_i)\right|^2
\end{align*}
Following the strategy used by \citet{Blanchard:11Generalizing} to control (II) in Theorem 5.1, we obtain
\begin{align*}
	(A) & \leq  c_3\frac{\phi_{\ell}^2U_\inspace^2 U_\pspace + U_{\ell}\log\delta^{-1}}{N}\cdot\|\bopt\|^2_{HS}.\\
%	& 	2\phi_l^2 L_\pspace U_\pspace B_{\inspace}\cdot\|\bopt\|^2_{HS}\cdot	\frac{1}{\ell}
\end{align*}
\textbf{End of proof:}
We have that $\kxmat$ is invertible since $\widehat{\Sigma}_{xx}$ is assumed to be invertible.  It follows that the trace $\text{tr}(\bmat^\top\kxmat\bmat)$ defines a norm which coincides with the Hilbert-Schmidt norm $\|\bopt\|^2_{HS}$. Combining the three inequalities above concludes the proof.
\end{proof}

\section{Leave-one-out accuracy}

%% leave-one-out results
\begin{figure}[t!]
  \centering
  \includegraphics[width=\linewidth]{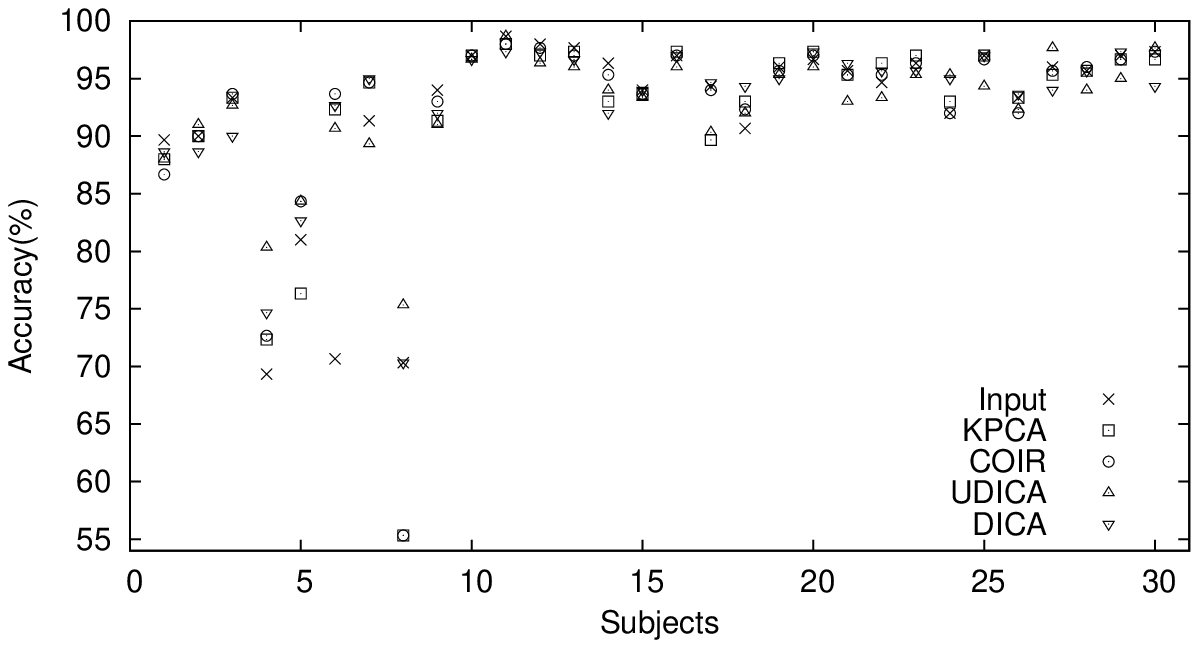}
  \includegraphics[width=\linewidth]{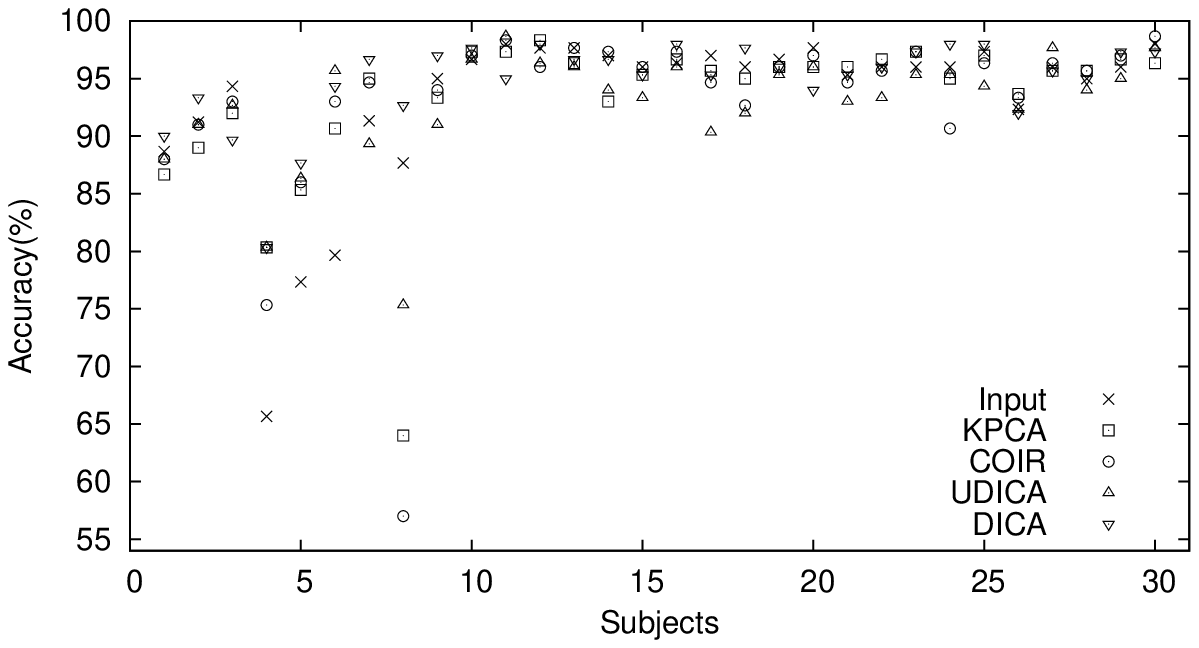}
  \caption{The leave-one-out accuracy of different methods evaluated on each subject in the GvHD dataset. The top figure depicts the pooling setting, whereas the bottom figure depicts the distributional setting.}
  \label{fig:flowloo}
\end{figure}

Figure \ref{fig:flowloo} depicts the leave-one-out accuracies of different approaches evaluated on each subject in the dataset. Average leave-one-out accuracies are reported in Table \ref{tab:flowloo}. The distributional SVM outperforms the pooling SVM in this setting, possibly because of the relatively large number of training subjects, i.e., 29 subjects. Using the invariant features learnt by DICA also gives higher accuracies than other approaches.

\end{document}